\PassOptionsToPackage{table,svgnames}{xcolor}
\documentclass[11pt,letterpaper]{mystyle}
\usepackage[comma,authoryear,compress]{natbib}
\hypersetup{colorlinks=true,citecolor=YaleBlue,
  pdftitle={Beyond One-Step Accuracy: State-Affine Latent Transition for Reliable Visual Planning},
  pdfauthor={Boyuan Zhang, Yingjun Du, Xiantong Zhen, Ling Shao}}
\titlespacing*{\paragraph}{0pt}{1.5ex plus 0.5ex minus 0.3ex}{0.6em}
\usepackage{amsmath,amsfonts,bm}

\def\eqref#1{equation~\ref{#1}}
\def\Eqref#1{Equation~\ref{#1}}

\def\1{\bm{1}}

\DeclareMathAlphabet{\mathsfit}{\encodingdefault}{\sfdefault}{m}{sl}
\SetMathAlphabet{\mathsfit}{bold}{\encodingdefault}{\sfdefault}{bx}{n}

\usepackage{hyperref}
\usepackage{url}

\usepackage{graphicx}

\usepackage{booktabs}
\usepackage{multirow}
\usepackage{xcolor}
\usepackage{arydshln}
\newlength\savedwidth
\newcommand\thickhline{\noalign{\global\savedwidth\arrayrulewidth
  \global\arrayrulewidth 1pt}\hline\noalign{\global\arrayrulewidth\savedwidth}}
\usepackage{tabularx}
\newcolumntype{Y}{>{\centering\arraybackslash}X}
\usepackage{wrapfig}
\usepackage{placeins}

\usepackage{amsmath}
\usepackage{amssymb}
\usepackage{mathtools}
\usepackage{amsthm}

\usepackage{xcolor}

\usepackage[capitalise]{cleveref}
\usepackage{subcaption}
\usepackage[normalem]{ulem}
\usepackage[misc]{ifsym}

 \usepackage{listings}
 \usepackage{algorithm}      
 \definecolor{kwcolor}{RGB}{110,60,140}
 \definecolor{fncolor}{RGB}{40,90,150}
 \definecolor{cmcolor}{RGB}{110,120,130}
 \definecolor{doccolor}{RGB}{190,90,30}

\lstdefinestyle{salt}{
  language=Python, basicstyle=\ttfamily\small, columns=fullflexible,
  keepspaces=true, showstringspaces=false,
  keywordstyle=\color{kwcolor}\bfseries,
  commentstyle=\color{cmcolor}\itshape,
  stringstyle=\color{doccolor},
  emph={encoder,action_encoder,affine_step,sigreg,mse,expm,tanh,diag,
        einsum,stack},
  emphstyle=\color{fncolor},
  morekeywords={def,for,in,range,return},
  backgroundcolor=\color{gray!6},
  frame=none,                 
  linewidth=\linewidth,       
  xleftmargin=0pt, xrightmargin=0pt,
  framexleftmargin=0pt,
  aboveskip=0.4em, belowskip=0.2em}

\newcommand{\ours}{\textsc{SALT}}

\newtheorem{proposition}{Proposition}

\graphicspath{{fig/}}

\title{Beyond One-Step Accuracy: State-Affine Latent Transition for Reliable Visual Planning}
\runningtitle{Beyond One-Step Accuracy: State-Affine Latent Transition for Reliable Visual Planning}
\author{%
  \textbf{Boyuan Zhang}$^{1}$ \quad
  \textbf{Yingjun Du}$^{2}$ \quad
  \textbf{Xiantong Zhen}$^{3}$ \quad
  \textbf{Ling Shao}$^{1,3,(\textrm{\Letter})}$ \\
  $^{1}$UCAS-Terminus AI Lab, School of Engineering Science, \\
  University of Chinese Academy of Sciences \\
  $^{2}$VIS Lab, University of Amsterdam \\
  $^{3}$OPI Labs \\
}
\correspondingauthor{Ling Shao ($^{(\textrm{\Letter})}$), \href{mailto:ling.shao@ieee.org}{\texttt{ling.shao@ieee.org}}}

\date{September 27, 2026}

\newcommand{\metadataicon}[1]{\makebox[1.5em][l]{\raisebox{-0.15em}{\includegraphics[height=1.1em]{icons/#1.pdf}}}}

\begin{document}
\begin{abstract}
Joint-embedding world models enable visual planning by learning
action-conditioned dynamics in latent space. 
Yet they are commonly trained for one-step prediction on encoded states,
while planning recursively applies the learned transition to its own predictions.
One-step accuracy therefore does not capture how prediction errors propagate
under recursive rollout.
We decompose multi-step rollout error into the errors introduced at individual steps and their propagation through subsequent transitions.
We show that state-affine dynamics are precisely the differentiable transitions
with state-independent Jacobians, eliminating the nonlinear propagation
residual and making the error propagation operators depend only on the action sequence.
Guided by this result, we introduce \ours{} (\textbf{S}tate-\textbf{A}ffine
\textbf{L}atent \textbf{T}ransition), an action-conditioned state-affine
dynamics model in which the action modulates both the state transformation
and the additive update.
We train \ours{} through recursive multi-step rollout supervision, feeding each
predicted latent state back into the transition so that training matches how the model is used during planning.
Across four visual planning environments, \ours{} exhibits
$1.48$--$2.19\times$ higher one-step prediction error than the LeWM
baseline, yet improves closed-loop success in every environment by
$10.0$ percentage points on average.
On OGBench-Cube, the fraction of episodes that fail with a sharp rise in model-predicted cost after execution decreases from $23.3\%$ to $2.0\%$.

\vspace{3mm}
\metadataicon{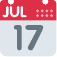}\textbf{Date:} September 27, 2026\par
\metadataicon{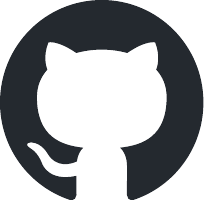}\textbf{Code Repository:} \url{https://github.com/deepmindby/SALT}\par
\metadataicon{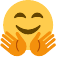}\textbf{Model Weights:} \url{https://huggingface.co/collections/ByDM/salt}\par
\metadataicon{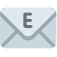}\textbf{Contact:}
\url{zhangboyuan23@mails.ucas.ac.cn}
\end{abstract}
\maketitle
\vspace{3mm}

\section{Introduction}
\label{sec:introduction}

Planning from visual observations requires a model that predicts how candidate actions will change the state of the environment. Joint-embedding world models learn these dynamics in a low-dimensional latent space without reconstructing pixels \citep{assran2023self,balestriero2025lejepa, maes_lelidec2026lewm}. In latent planning, a planner evaluates candidate action sequences by applying the learned transition model recursively \citep{hafner2019learning,sobal2026learning,zhou2025dinowm}. Training, however, often relies on a one-step prediction objective defined on encoded states, as in LeWM~\citep{maes_lelidec2026lewm}. Such supervision does not constrain how prediction errors propagate when the model recursively uses its own predictions. A more accurate one-step predictor can therefore be a less effective planning model.

During recursive rollout, each transition introduces a new prediction error, while earlier errors are transformed by later transitions \citep{asadi2019combating,somalwar2025learning}. One-step prediction error measures the deviation introduced by a transition from an encoded state, but does not describe how subsequent transitions transform that deviation. Over multiple steps, rollout error therefore depends on both the errors introduced along the trajectory and their propagation through the learned dynamics. The challenge is therefore to learn latent dynamics that remain effective for planning as prediction errors are fed back through the model.

To study this effect, we decompose multi-step rollout error into errors introduced at individual steps and their subsequent propagation. For nonlinear state transitions, this propagation involves state-dependent Jacobians and residual terms beyond the first-order approximation. We therefore ask which transition structure makes the propagation operator independent of the rollout state. We show that, among differentiable transitions, this property holds precisely when the predictor is affine in the latent state. In this case, the nonlinear residual vanishes, and the propagation operators are expressed exactly through products of action-conditioned transition matrices. For recursive planning, this provides a way to separate the accuracy of individual predictions from how their errors are transformed along a candidate action sequence, motivating a structural approach to transition design.

Guided by this analysis, we introduce \ours{}, an action-conditioned transition model that remains affine in the latent state. The action modulates both the state transformation and the additive update, allowing different actions to induce different dynamics while keeping the state Jacobian independent of the latent state. This structural choice determines the form of error propagation, but does not by itself align one-step supervision with recursive prediction. We therefore train \ours{} with a multi-step rollout objective, feeding each predicted latent state back into the transition and supervising the resulting trajectory. The state-affine structure determines the form of error propagation, while rollout training aligns supervision with the recursive feedback used in planning.

Across four visual planning environments, \ours{} has $1.48$--$2.19\times$ the one-step prediction error of the matched LeWM reproduction, yet achieves a higher mean closed-loop success rate on each environment. Its average success rate reaches $92.7\%$, exceeding the baseline by $10.0$ percentage points. These results demonstrate that lower one-step prediction error alone does not determine better planning performance. On OGBench-Cube, the rate of failures characterized by sharply higher model-predicted costs after execution and re-observation decreases from
$23.3\%$ of episodes for LeWM to $2.0\%$ for \ours{}. In these cases, plans look promising before execution but no longer do once the model sees the result.
\section{Related Work}
\label{sec:related_work}

\textbf{Latent World Models for Visual Planning.}
Latent world models support planning and control by learning compact predictive representations of visual observations. Early work learned latent state-space dynamics for prediction and control from images \citep{watter2015embed,karl2017deep}. World Models, PlaNet, and subsequent latent-imagination approaches further established learned latent dynamics as an effective basis for control from pixels \citep{ha2018recurrent,hafner2019learning,hafner2021mastering,hafner2025mastering}. More recent methods target visual planning directly. PLDM learns latent dynamics from reward-free offline data \citep{sobal2026learning}, DINO-WM plans over pretrained visual features \citep{zhou2025dinowm}, LeWorldModel (LeWM) jointly learns visual representations and action-conditioned latent dynamics \citep{maes_lelidec2026lewm}, and TD-MPC2 optimizes trajectories in learned latent dynamics \citep{hansen2024tdmpc2}. Related work examines reliability and efficiency in closed-loop use \citep{duan2024learning,zhang2026worldinworld,chun2026sparse}, while large-scale video JEPA models learn predictive representations for physical understanding and planning \citep{assran2025v}. These advances demonstrate the value of latent prediction for planning, but planning performance alone does not characterize how errors propagate through repeated applications of the learned transition. We characterize how repeated composition of the learned transition transforms prediction errors, separating the errors introduced at each step from their transformation by subsequent transitions.

\textbf{Stable and Structured Joint-Embedding World Models.}
Joint-embedding predictive learning aims to learn stable and structured representations without pixel-level reconstruction. I-JEPA learns image representations by predicting target-block embeddings from a context block \citep{assran2023self}. LeJEPA introduced SIGReg to stabilize joint-embedding learning through distributional regularization \citep{balestriero2025lejepa}, and LeWM applies it to end-to-end visual world modeling \citep{maes_lelidec2026lewm}. Recent extensions mainly modify representation learning through quantile matching, subspace regularization, latent decomposition, or inverse-dynamics supervision \citep{yu2026qqworld,zhao2026subjepa,thil2026sdjepa,ivashkov2026sensorimotor}, with related objectives also studied outside the joint-embedding setting \citep{sun2024learning,wang2025dyno}. These representation objectives do not by themselves enforce a state-independent transition Jacobian. Beyond representation learning, \citet{terver2026drivessuccessphysicalplanning} study training and planning
design choices for joint-embedding world models and analyze autoregressive error growth using Lipschitz bounds. We focus on the structural condition under which the state Jacobian is independent of the latent state, yielding an exact error decomposition for action-conditioned state-affine transitions.

\textbf{Multi-Step Prediction and Structured Latent Dynamics.}
Multi-step learning extends prediction supervision along latent trajectories \citep{karl2017deep,hafner2019learning}. Work on compounding model error further studies how multi-step objectives, temporal abstraction, and uncertainty-aware rollouts affect prediction over longer horizons \citep{asadi2019combating,somalwar2025learning,gumbsch2024learning,wang2024coplanner,wang2025dmwm}. However, supervising longer rollouts does not by itself enforce a state-independent transition Jacobian. Structured dynamics constrain the transition directly. Locally linear models allow the transition matrices to vary with the state \citep{watter2015embed}, while softly state-invariant models regularize the dependence of latent state changes on the current state \citep{saanum2024simplifying}. Koopman methods seek representations with linear or bilinear dynamics for prediction and control \citep{brunton2016koopman,takeishi2017learning,lusch2018deep,shi2022deep,huang2023learning,zhao2024deep,chen2024korol}. For fixed actions, linear and bilinear transitions have state-independent Jacobians and belong to the state-affine family. We derive this family from the Jacobian requirement and express its rollout error exactly through action-conditioned propagation matrices. We instantiate this family in a visual joint-embedding world model, learning the encoder and transition jointly from pixels through recursive rollout supervision.
\section{Method}
\label{sec:method}

\subsection{Problem Setup}
\label{sec:problem-setup}

A joint-embedding predictive architecture models future dynamics in a
low-dimensional latent space rather than directly in pixel space.
Let $o_t$ denote the image observation at time $t$ and
$a_t\in\mathbb{R}^{d_a}$ the control input.
The encoder $\phi$, including its projection head, maps an observation to the
latent state $z_t=\phi(o_t)\in\mathbb{R}^{d}$.
The action encoder $\psi$ maps the control input to an action condition
$c_t=\psi(a_t)\in\mathbb{R}^{d}$.
The predictor $F$ predicts the next latent state from $(z_t,c_t)$.
We denote by $\hat z_t$ a latent state generated during autoregressive rollout.

A common training objective for latent transition models is one-step prediction:
\begin{equation}
  \mathcal{L}_{\text{pred}}
  =
  \mathbb{E}_t
  \big\|
  F(z_t,c_t)-z_{t+1}
  \big\|^2.
  \label{eq:pred}
\end{equation}
The predictor is evaluated at latent states encoded from observed trajectories.
\Eqref{eq:pred} therefore directly supervises transitions from encoded states
rather than states generated by previous predictions.

At planning time, the transition model is applied recursively.
Given the current latent state $z_0$ and a goal observation $o_g$ with embedding
$z_g=\phi(o_g)$, the planner optimizes an action sequence over a horizon of
$H$ steps:
\begin{equation}
  \begin{aligned}
  a^*_{0:H-1}
  =\;&
  \arg\min_{a_{0:H-1}}
  \ \big\|\hat z_H-z_g\big\|^2
  \\
  \text{subject to}\quad
  &
  \hat z_0=z_0,
  \qquad
  \hat z_k=F(\hat z_{k-1},c_{k-1}),
  \qquad
  c_{k-1}=\psi(a_{k-1}),
  \quad
  k=1,\dots,H.
  \end{aligned}
  \label{eq:plan}
\end{equation}
We optimize \eqref{eq:plan} using the cross-entropy method (CEM). Each iteration samples $S$ candidate action sequences, evaluates their latent rollouts, and updates the sampling distribution using the sequences with the lowest terminal costs. No new observation is available during a model rollout to correct the predicted states. After executing the selected sequence, the planner replans from the new observation.

\Eqref{eq:pred} evaluates individual transitions from encoded states, whereas planning depends on the composition $F(\cdot,c_{H-1})\circ\cdots\circ F(\cdot,c_0)$ along a predicted trajectory. The resulting rollout error depends on both the errors introduced at individual steps and their propagation through subsequent transitions.

\subsection{Recursive Error Propagation}
\label{sec:error-propagation}

We distinguish two errors during recursive rollout.
The \emph{one-step error}
$\varepsilon_t=F(z_{t-1},c_{t-1})-z_t$
is the deviation incurred by a single transition from the encoded state
$z_{t-1}$.
Its squared norm is the per-transition loss in \eqref{eq:pred}.
The \emph{rollout error}
$\delta_t=\hat z_t-z_t$
measures the accumulated deviation from the encoded reference trajectory after
$t$ recursive prediction steps, starting from $\hat z_0=z_0$.
Substituting
$\hat z_{t-1}=z_{t-1}+\delta_{t-1}$
into the definition of $\delta_t$ and expanding the predictor around the
encoded state $z_{t-1}$ gives
\begin{equation}
  \begin{aligned}
    \delta_t
    &=
    F(\hat z_{t-1},c_{t-1})-z_t
    \\
    &=
    F(z_{t-1}+\delta_{t-1},c_{t-1})
    -
    F(z_{t-1},c_{t-1})
    +
    F(z_{t-1},c_{t-1})
    -
    z_t
    \\
    &=
    J_t\,\delta_{t-1}
    +
    r_t
    +
    \varepsilon_t,
  \end{aligned}
  \label{eq:delta-recursion}
\end{equation}
where
$J_t=\partial_zF(z_{t-1},c_{t-1})$
is the state Jacobian evaluated at $z_{t-1}$, and $r_t$ is the residual beyond
the first-order linearization.

\begin{proposition}[Multi-step error decomposition]
\label{prop:error_decomposition}
Let $F$ be any predictor that is differentiable in the state.
Then, for a rollout of $H$ steps from $\hat z_0=z_0$, the rollout error can be
written as
\begin{equation}
  \delta_H
  =
  \sum_{j=1}^{H}
  \Phi_{j\to H}
  \big(
  \varepsilon_j+r_j
  \big),
  \qquad
  \Phi_{j\to H}
  =
  J_HJ_{H-1}\cdots J_{j+1},
  \qquad
  \Phi_{H\to H}=I.
  \label{eq:decomp}
\end{equation}
If, in addition, $\partial_{z}F(\cdot, c)$ is $L$-Lipschitz in the state with the same constant $L$ for every action condition $c$, then the residual satisfies $\|r_t\|\le\frac{L}{2}\|\delta_{t-1}\|^2$.
\end{proposition}

\noindent
The proof is given in Appendix~\ref{app:decomposition}.
The operator $\Phi_{j\to H}$ describes how an error introduced at step $j$ is
subsequently transformed by all later transitions before reaching step $H$.

\Eqref{eq:decomp} separates two factors in multi-step prediction:
the error introduced by individual transitions and the propagation of those
errors through later transitions.

For nonlinear transitions, $r_j$ need not vanish, and the Jacobians in $\Phi_{j\to H}$ can vary with the encoded reference states at which they are evaluated. The propagation operator can therefore depend on the state trajectory as well as the action sequence.

We examine this state dependence in LeWM \citep{maes_lelidec2026lewm} by fixing the action sequence and varying the current latent state. Since LeWM conditions on a short latent history, we differentiate with respect to its most recent latent input while holding the earlier history inputs fixed when computing each Jacobian. \Cref{fig:jacobian-state} shows that the largest singular value of this Jacobian varies substantially with the current latent state.

\begin{figure}[t]
  \centering
  \begin{subfigure}[t]{0.48\linewidth}
    \centering
    \includegraphics[width=\linewidth]{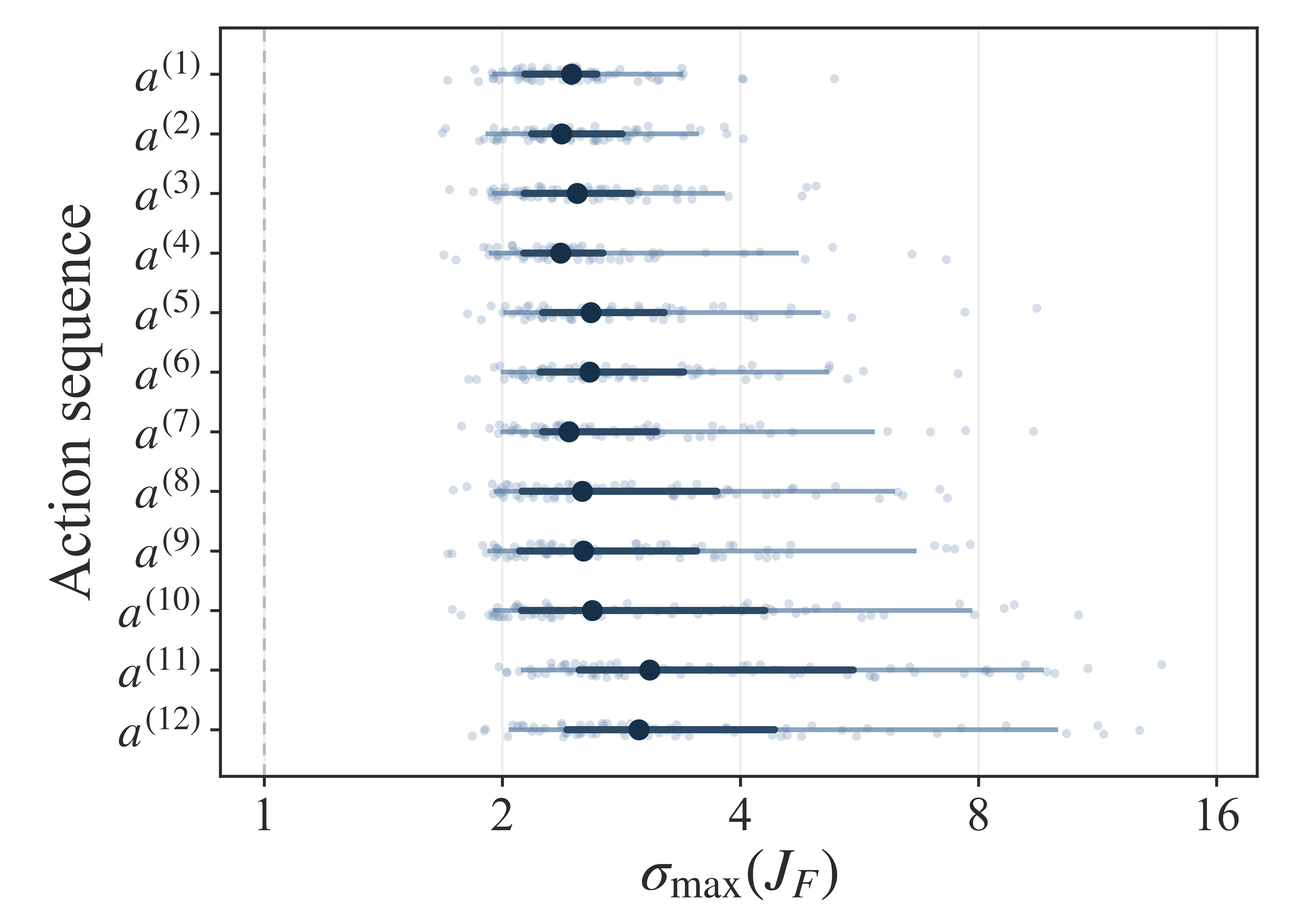}
    \caption{State dependence of the Jacobian.}
    \label{fig:jacobian-state}
  \end{subfigure}
  \hfill
  \begin{subfigure}[t]{0.48\linewidth}
    \centering
    \includegraphics[width=\linewidth]{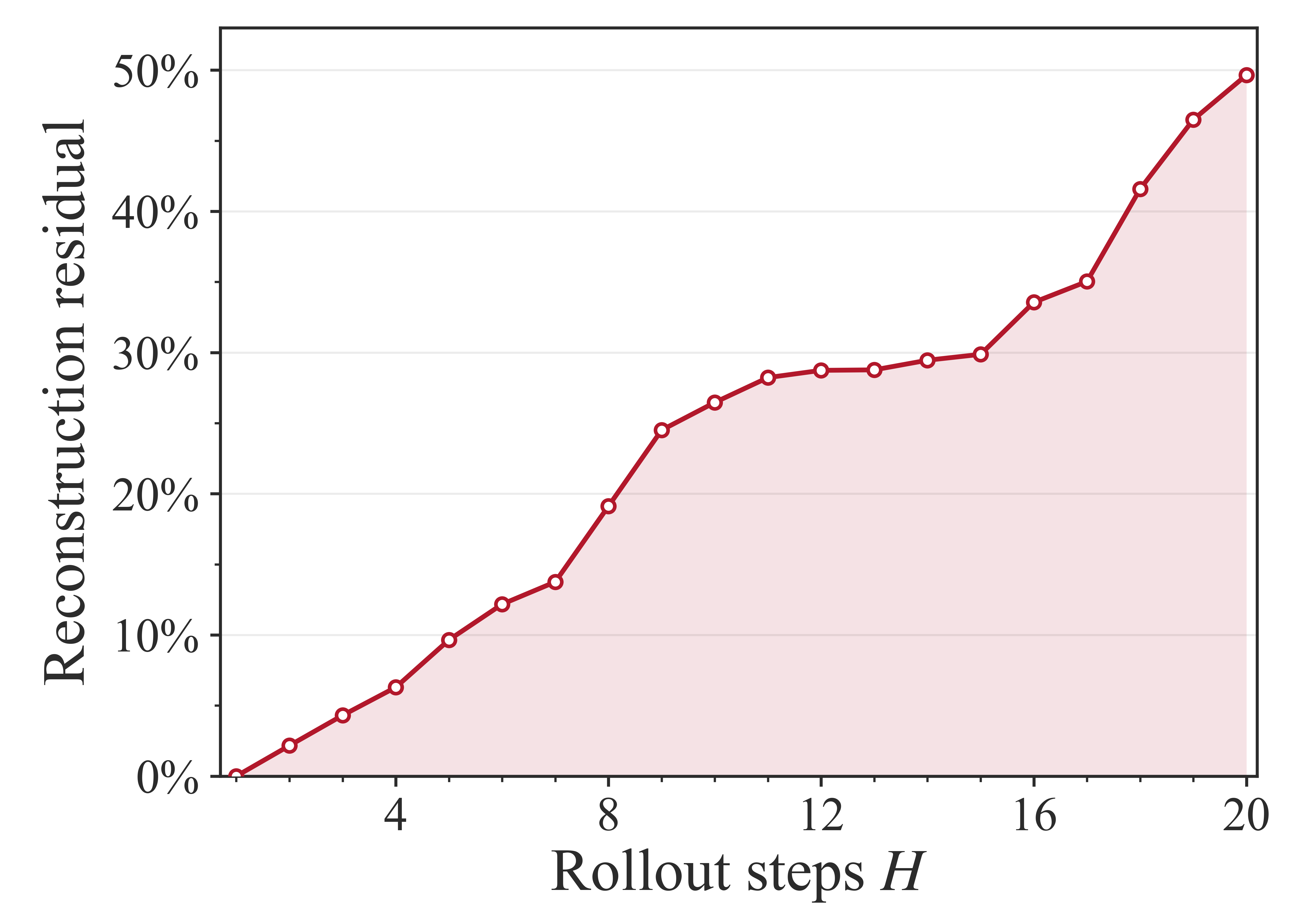}
    \caption{Reconstruction discrepancy versus horizon.}
    \label{fig:jacobian-residual}
  \end{subfigure}
  \caption{
  \textbf{State dependence and current-state rollout reconstruction of the
  LeWM predictor.}
  (a) Distribution of $\sigma_{\max}(\partial_zF)$ for the current-state
  Jacobian across 12 fixed action sequences while varying the current latent
  state.
  Points denote samples, the thick bar shows the interquartile range, and the
  dashed line marks $\sigma_{\max}=1$.
  (b) Relative discrepancy between the rollout error and its current-state
  first-order reconstruction as the rollout horizon increases. The reconstruction uses only current-state Jacobians and omits propagation through earlier history inputs.
  }
  \label{fig:jacobian}
\end{figure}

For this diagnostic, we form $\Phi_{j\to H}$ from products of these current-state Jacobians and measure the relative reconstruction discrepancy
\begin{equation}
  \rho_H = \frac{\left\|\delta_H-\sum_{j=1}^{H}\Phi_{j\to H}\varepsilon_j\right\|}{\|\delta_H\|}.
\end{equation}
As shown in \Cref{fig:jacobian-residual}, this discrepancy increases from about $10\%$ at $H=5$ to about $49\%$ at $H=20$. The reconstruction excludes propagation through earlier history inputs, so the discrepancy can contain both nonlinear effects in the current-state channel and contributions from the omitted history channels. Thus, the current-state linear reconstruction leaves an increasing fraction of the rollout error unexplained over the tested horizons.

\subsection{State-Affine Dynamics}
\label{sec:state-affine}

We seek a transition family whose propagation operators are determined by the action sequence at every rollout step. We therefore require the state Jacobian to be independent of the latent state. For differentiable transitions, this requirement is equivalent to an affine dependence on the latent state.

\begin{proposition}[Characterization of state-independent Jacobians]
\label{prop:state-affine}
For each fixed action condition $c$, let $F(\cdot,c):\mathbb R^d\to\mathbb R^d$ be differentiable. Its state Jacobian is independent of $z$ if and only if $F(\cdot,c)$ is affine:
\begin{equation}
  F(z,c)
  =
  \mathcal{A}(c)\,z
  +
  \mathcal{B}(c).
  \label{eq:affine}
\end{equation}
\end{proposition}

\noindent
The proof is given in Appendix~\ref{app:affine-characterization}.

Under \eqref{eq:affine}, any state deviation $\delta$ is propagated exactly as
\[
F(z+\delta,c)-F(z,c)
=
\mathcal{A}(c)\delta.
\]
Consequently, \eqref{eq:delta-recursion} gives
$r_t\equiv0$ and
$J_t=\mathcal{A}(c_{t-1})$.
The multi-step error decomposition therefore reduces to
\begin{equation}
  \delta_H
  =
  \sum_{j=1}^{H}
  \Phi_{j\to H}\,
  \varepsilon_j,
  \qquad
  \Phi_{j\to H}
  =
  \mathcal{A}(c_{H-1})
  \mathcal{A}(c_{H-2})
  \cdots
  \mathcal{A}(c_j).
  \label{eq:decomp-affine}
\end{equation}

For a fixed action sequence, the propagation matrices in \eqref{eq:decomp-affine} can be computed directly from $\mathcal A(c_k)$ without generating the latent rollout. We use this property for analysis and diagnostics, while planning uses the CEM procedure defined in \Cref{sec:problem-setup}.

\subsection{Action-Conditioned State-Affine Transition}
\label{sec:action-conditioned}

\Cref{prop:state-affine} requires an affine dependence on the latent state but does not prescribe the functional forms of $\mathcal{A}(c)$ and $\mathcal{B}(c)$. We use the following parameterization:
\begin{equation}
  F(z,c)
  =
  \mathcal{A}(c)\,z
  +
  Bc+b,
  \qquad
  \mathcal{A}(c)
  =
  \mathcal{A}_0
  +
  \sum_{r=1}^{R}
  \big(W_gc\big)_r\,N_r,
  \label{eq:transition}
\end{equation}
where $\mathcal{A}_0,N_r,B\in\mathbb{R}^{d\times d}$, $W_g\in\mathbb{R}^{R\times d}$, and $b\in\mathbb{R}^{d}$. The coefficient $(W_gc)_r$ weights the modulation matrix
$N_r$, while $Bc+b$ parameterizes $\mathcal{B}(c)$.
This allows actions to change the state transformation itself while preserving affine dependence on the latent state. \Cref{app:action-ablation} evaluates the role of this action-dependent modulation.

We constrain $\sigma_{\max}(\mathcal{A}_0)<1$ to limit the gain of the shared base transition. \Cref{app:implementation} provides the parameterization and initialization details.

\subsection{Recursive Rollout Training}
\label{sec:training}

The state-affine structure makes recursive error propagation explicit, but it
does not by itself resolve the mismatch between one-step supervision and the
recursive use of the transition at planning time.
We therefore train \ours{} with a recursive multi-step rollout objective.

Let $K$ denote the training rollout length, which is distinct from the planning
horizon $H$.
Starting from $\hat z_0=z_0$, the transition is applied recursively as
\[
\hat z_k
=
F(\hat z_{k-1},c_{k-1}),
\]
and every predicted state is supervised:
\begin{equation}
  \mathcal{L}_{\text{roll}}
  =
  \mathbb{E}
  \left[
  \frac{1}{K}
  \sum_{k=1}^{K}
  \big\|
  \hat z_k-z_k
  \big\|^2
  \right].
  \label{eq:roll}
\end{equation}
During training, predicted latents are fed back into the transition, matching the recursive use of the model during planning. We use $K=5$, supervising all five predictions from a six-frame training window.

The training objective combines recursive rollout supervision with SIGReg:
\begin{equation}
  \mathcal{L}
  =
  \mathcal{L}_{\text{roll}}
  +
  \lambda\,\mathcal{R}_{\text{SIGReg}},
  \label{eq:total}
\end{equation}
Here, $\mathcal{R}_{\text{SIGReg}}$ regularizes the encoded latents from all frames in the training window, with $\lambda=0.09$. SIGReg discourages representation collapse by encouraging the latent distribution to match an isotropic Gaussian \citep{balestriero2025lejepa}. We jointly optimize the encoder, action encoder, and transition, retaining gradients through the target latents as in LeWM \citep{maes_lelidec2026lewm}.

The state-affine structure determines the form of error propagation, while rollout supervision trains the model under the recursive feedback used in planning.
\section{Experiments}
\label{sec:experiments}

\subsection{Experimental Setup}
\label{sec:setup}

We evaluate \ours{} on four visual planning environments, following the
benchmark setting of LeWM \citep{maes_lelidec2026lewm}: Two-Room
\citep{sobal2025stresstesting}, Reacher
\citep{tassa2018deepmindcontrolsuite}, PushT
\citep{chi2023diffusionpolicy,zhou2025dinowm}, and OGBench-Cube
\citep{ogbench_park2025}.
Planning is performed with CEM \citep{rubinstein2004cross}, using the MSE
between the predicted terminal latent state and the goal embedding as the
terminal cost.
The planning horizon and the replanning interval are both set to $5$, so the
planner replans from a new observation only after the entire planned sequence
has been executed.
For each environment, we evaluate one trained model per method on three evaluation seeds of $50$ episodes each, and report the mean and standard deviation across seeds. Each goal is the observation $25$ environment steps after the initial state in the same offline trajectory.
All experiments conducted in our implementation use a single NVIDIA H100 $80$\,GB GPU.
Additional environment, training, and evaluation details are provided in Appendix~\ref{app:detail_setup}.

For controlled comparison with the reproduced baseline, the encoder
architecture, action encoder, representation regularizer, and planner are
matched between the two models. The models differ in their transition structure and prediction training configuration, whose combinations are evaluated in \Cref{tab:objective}.

We evaluate Random, reproduced LeWM, and \ours{} under a shared protocol. Reproduced LeWM serves as the controlled baseline for all paired comparisons. Results for PLDM \citep{sobal2026learning}, DINO-WM \citep{zhou2025dinowm}, Sub-JEPA \citep{zhao2026subjepa}, SD-JEPA \citep{thil2026sdjepa}, and SMWM \citep{ivashkov2026sensorimotor} are taken from their respective papers and reported separately in \Cref{tab:main} to provide context.

\subsection{Closed-Loop Planning Performance}
\label{sec:planning-performance}

\Cref{tab:main} reports closed-loop planning success rates on the four
environments.

\begin{table}[!htbp]
\centering
\small
\setlength\tabcolsep{3pt}
\renewcommand\arraystretch{1.3}
\caption{\textbf{Closed-loop planning success rates (\%) across four environments.} Controlled evaluations use $H=5$ and report the mean$\pm$standard deviation over three evaluation seeds of 50 episodes each. $\dagger$ quoted from the literature under different protocols. Bold marks the best mean among controlled rows.}
\label{tab:main}
\begin{tabularx}{\linewidth}{r||YYYY|Y}
\hline\thickhline
\rowcolor{gray!20}
Method & Two-Room & Reacher & PushT & \mbox{OGBench-Cube} & Avg. \\
\hline\hline
\rowcolor{gray!10}
PLDM$^\dagger$     & 97 & 78 & 78 & 65 & 79.50 \\
\rowcolor{gray!10}
DINO-WM$^\dagger$  & 100 & 79 & 74 & 86 & 84.75 \\
\rowcolor{gray!10}
Sub-JEPA$^\dagger$ & 95.00 & 84.00 & 89.00 & 76.33 & 86.08 \\
\rowcolor{gray!10}
SD-JEPA$^\dagger$  & 90.00 & 88.00 & 97.30 & 72.00 & 86.83 \\
\rowcolor{gray!10}
SMWM$^\dagger$     & 99.00 & 66.00 & 83.00 & 84.00 & 83.00 \\
\hline
Random   & 2.00{\scriptsize$\pm$2.00} & 12.00{\scriptsize$\pm$4.00} & 5.30{\scriptsize$\pm$2.30} & 42.00{\scriptsize$\pm$4.00} & 15.33 \\
LeWM     & 90.67{\scriptsize$\pm$4.60} & 85.33{\scriptsize$\pm$2.30} & 87.33{\scriptsize$\pm$4.71} & 67.33{\scriptsize$\pm$4.20} &  82.67\\
\hline
\rowcolor[HTML]{D7F6FF}
\textbf{\ours{}} & \textbf{98.00}{\scriptsize$\pm$2.00} & \textbf{90.70}{\scriptsize$\pm$3.10} & \textbf{88.00}{\scriptsize$\pm$5.20} & \textbf{94.00}{\scriptsize$\pm$2.00} & \textbf{92.68} \\
\hline\thickhline
\end{tabularx}
\end{table}

\textbf{Planning Performance Across Four Environments.}
\ours{} reaches an average success rate of $92.7\%$, exceeding reproduced LeWM by $10.0$ percentage points under the matched evaluation protocol. Its mean success rate is higher in all four environments. The largest gain occurs on OGBench-Cube, where success increases from $67.3\%$ to $94.0\%$, a gain of $26.7$ percentage points.

\paragraph{Inference and Planning Efficiency.}
\Cref{tab:efficiency} compares model size and computational efficiency. Counting the predictor and its projection head where present, \ours{} uses $0.70$\,M parameters, compared with $11.58$\,M for LeWM. Relative to LeWM, \ours{} achieves an $8.1\times$ speedup in single-step prediction and a $3.1\times$ speedup in CEM planning. This compact transition reduces the cost of repeatedly evaluating candidate action sequences, a central computation in model-based planning.

\begin{table}[!htbp]
\centering
\small
\setlength\tabcolsep{3pt}
\renewcommand\arraystretch{1.3}
\caption{\textbf{Predictor complexity and computational efficiency.} Parameter counts and single-step forward times cover the predictor together with its projection head. Planning time is the CEM wall-clock time averaged over the four environments and $H\in\{5,10,15,20\}$. Additional timing details are provided in Appendix~\ref{app:efficiency}.}
\label{tab:efficiency}
\begin{tabularx}{\linewidth}{r||YY|YY|Y}
\hline\thickhline
\rowcolor{gray!20}
 & & & \multicolumn{2}{c|}{Training (it/s)} & \\
\cline{4-5}
\rowcolor{gray!20}
\multirow{-2}{*}{Predictor} & \multirow{-2}{*}{Params}
& \multirow{-2}{*}{\shortstack{Forward (ms)}}
& One-step & Rollout & \multirow{-2}{*}{\shortstack{Planning (s)}} \\
\hline\hline
LeWM (ViT-S)           & 11.58M & 2.36 & 6.5 & 2.9 & 35.2 \\
\rowcolor[HTML]{D7F6FF}
\textbf{\ours{} (Affine)} & 0.70M {\scriptsize\textcolor{red}{\textbf{($\times$16.5)}}} & 0.29 {\scriptsize\textcolor{red}{\textbf{($\times$8.1)}}}
                       & 6.0 & 2.4 & 11.4 {\scriptsize{\textcolor{red}{\textbf{($\times$3.1)}}}} \\
\hline\thickhline
\end{tabularx}
\end{table}

\subsection{Why One-Step Accuracy Fails to Predict Planning Quality}
\label{sec:better-planner}

\ours{} achieves better closed-loop planning performance than LeWM despite
having a higher one-step prediction error.
We now ask what explains this apparent contradiction.
\Cref{fig:mechanism} examines three increasingly deployment-relevant
quantities: one-step prediction error, recursive rollout error, and the
propagation of errors through subsequent transitions.

\paragraph{One-Step Error Is Not a Reliable Proxy for Planning Quality.}
As shown in \Cref{fig:mech-epsilon}, the one-step prediction error of \ours{}
is $1.48\times$--$2.19\times$ that of LeWM across the four environments, yet
\ours{} achieves a higher mean closed-loop success rate in every case.
Moreover, the relative one-step error does not track the planning gain:
Two-Room has the largest relative prediction error but only an intermediate
planning gain, whereas OGBench-Cube has one of the smallest relative errors and
the largest gain.
Thus, in the evaluated setting, one-step prediction error alone is not a
reliable proxy for closed-loop planning quality.

\paragraph{Rollout Error Alone Does Not Explain the Planning Gap.}
As shown in \Cref{fig:mech-error}, the rollout-error ranking of the two models
reverses during recursive prediction.
On PushT, \ours{} has higher rollout error than LeWM over the first four steps
and lower error thereafter.
Over twenty steps, rollout error increases by factors of $6.29$ and $12.83$
for \ours{} and LeWM, respectively.
The four-environment measurements in \Cref{app:four-env-propagation} further
show that the models have similar rollout errors on OGBench-Cube over much
of the evaluated horizon, despite the $26.7$ percentage point planning gap.
These observations indicate that rollout-error magnitude alone does not
fully characterize the difference in planning performance.

\paragraph{Evidence for Milder Recursive Error Propagation.}
As shown in \Cref{fig:mech-phi}, the measured propagation-operator norms on
PushT are comparable at the first step but grow at different rates thereafter.
Over twenty steps, the norm increases by a factor of $4.28$ for \ours{} and
$8.74$ for LeWM, and the ranges over nine measurements no longer overlap after
the fourth step.
Consistent with the exact state-affine decomposition in
\eqref{eq:decomp-affine}, the relative reconstruction discrepancy of \ours{}
remains at the $10^{-6}$ level across rollout steps.
In contrast, the current-state first-order reconstruction of LeWM becomes
increasingly incomplete with horizon, as detailed in Appendix~\ref{app:error-validation}. These diagnostics are consistent with milder propagation in the measured state channel, alongside the planning advantage of \ours{}.

\begin{figure*}[t]
  \centering
  \begin{subfigure}[t]{0.32\linewidth}
    \centering
    \includegraphics[width=\linewidth]{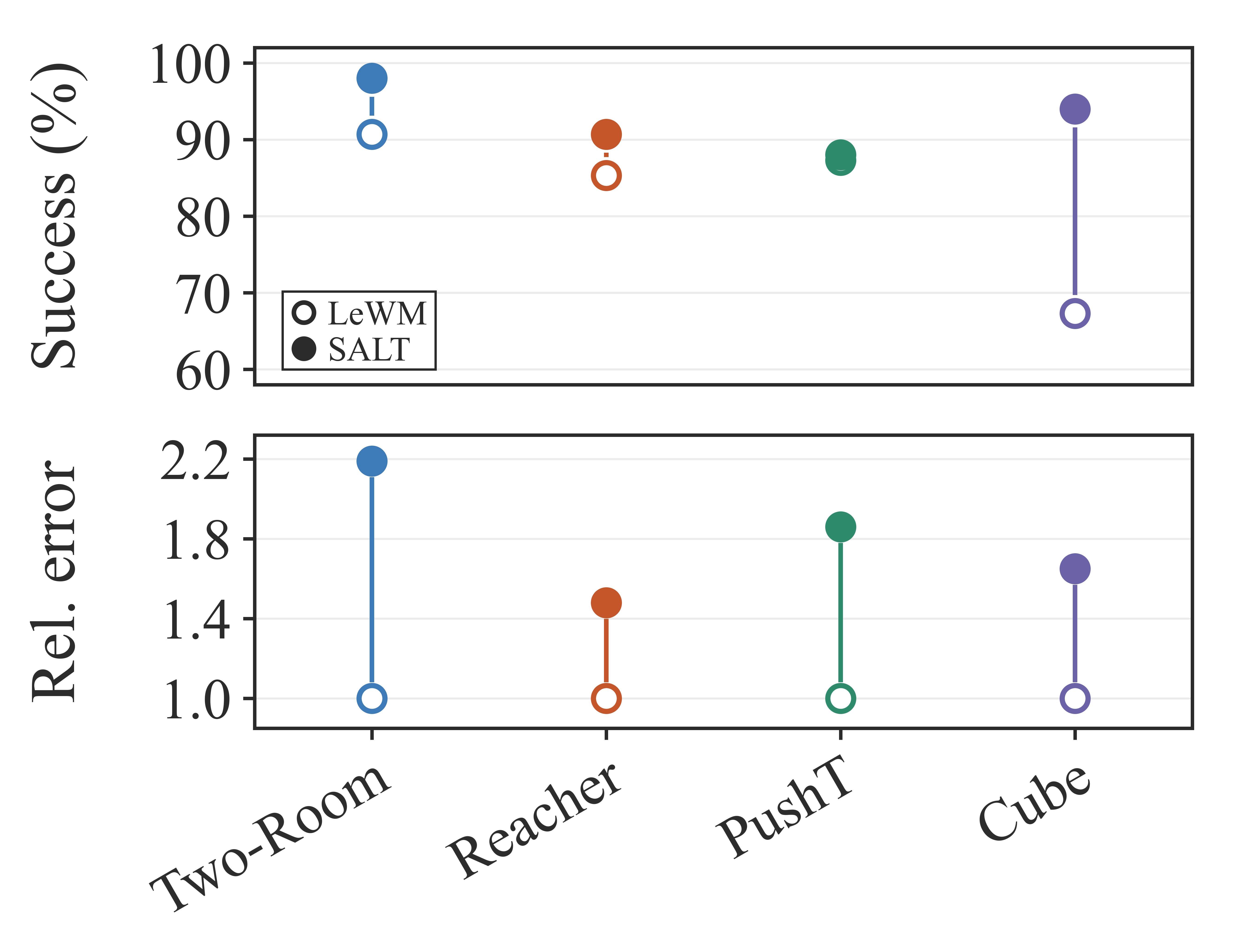}
    \caption{One-step error vs.\ success.}
    \label{fig:mech-epsilon}
  \end{subfigure}
  \hfill
  \begin{subfigure}[t]{0.32\linewidth}
    \centering
    \includegraphics[width=\linewidth]{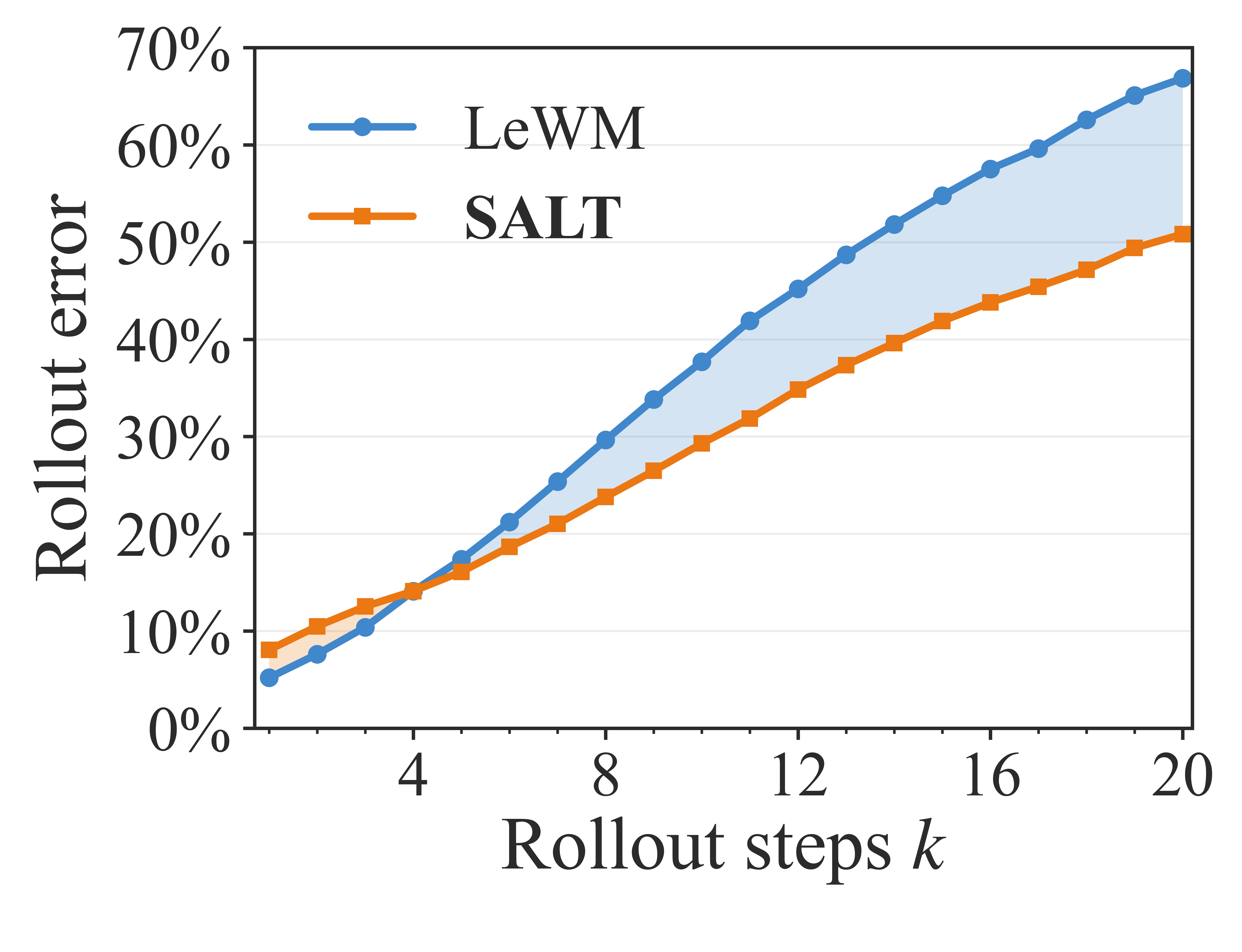}
    \caption{Rollout error growth.}
    \label{fig:mech-error}
  \end{subfigure}
  \hfill
  \begin{subfigure}[t]{0.32\linewidth}
    \centering
    \includegraphics[width=\linewidth]{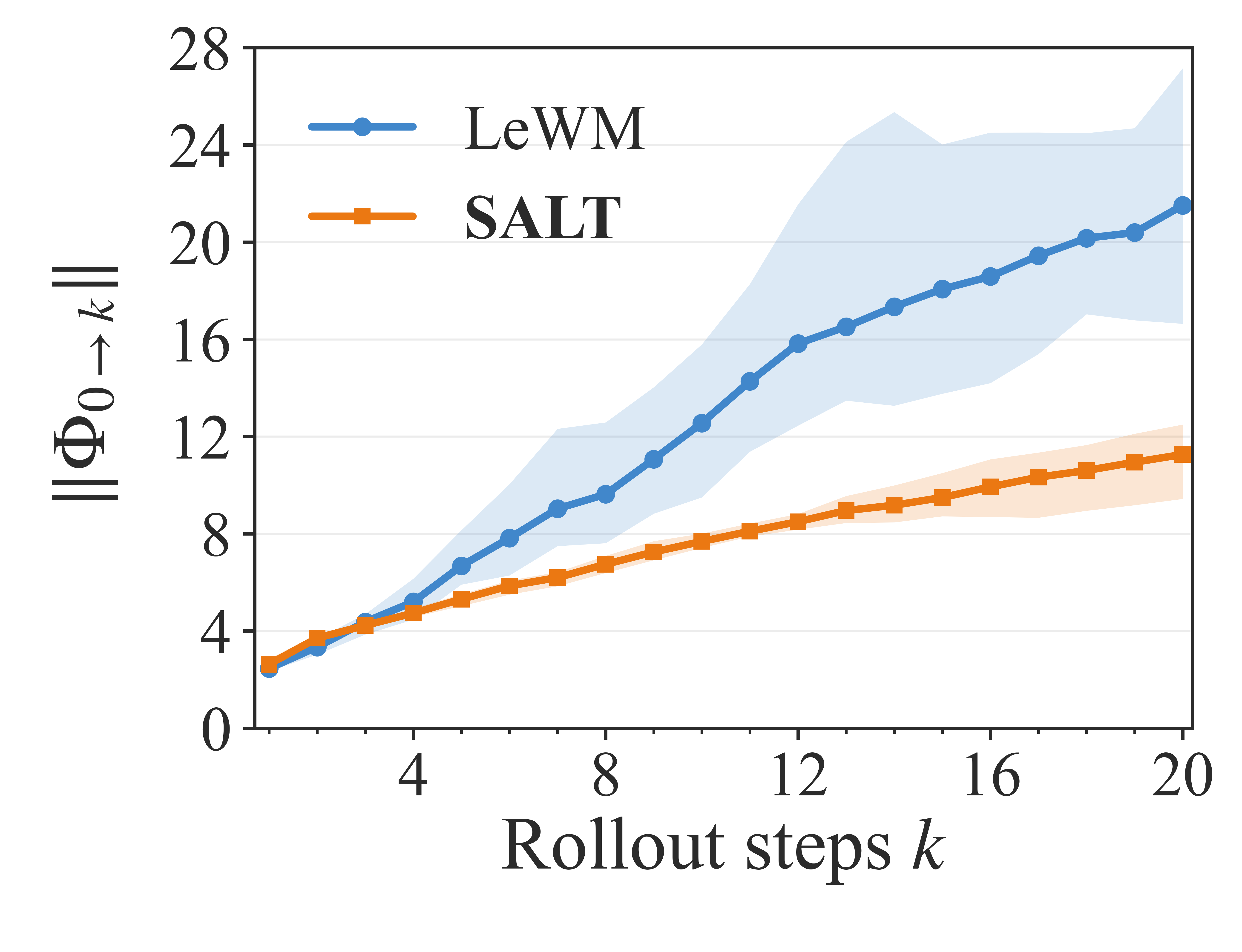}
    \caption{Propagation-operator norm.}
    \label{fig:mech-phi}
  \end{subfigure}
  \caption{
  \textbf{From one-step accuracy to recursive error propagation.}
  (a) Relative one-step prediction error and closed-loop success across four
  environments.
  (b) Recursive rollout error on PushT.
  (c) Spectral norm of the measured propagation operator on PushT; the band
  spans the minimum and maximum over nine measurements.
  }
  \label{fig:mechanism}
  \vspace{-6pt}
\end{figure*}

\paragraph{Interaction Between Transition Structure and Rollout Training.}
\Cref{tab:objective} reports mean success across the four environments for
the four combinations of transition structure and training objective, with
matched training windows, target frames, and training budgets.
Rollout training increases the mean success of \ours{} from $84.4\%$ to
$92.7\%$, whereas the LeWM mean decreases from $82.7\%$ to $77.8\%$.
The full configuration achieves the highest mean, indicating that the
benefit of recursive rollout supervision depends on the transition structure.

\begin{table}[!htbp]
\centering
\small
\setlength\tabcolsep{2pt}
\renewcommand\arraystretch{1.3}
\caption{\textbf{Transition structure and training objective.}
Mean closed-loop success rate (\%) over Two-Room, Reacher, PushT, and
OGBench-Cube, with equal weight assigned to each environment.
Training windows, target frames, and training budgets are matched across objectives.}
\label{tab:objective}
\begin{tabularx}{\linewidth}{r||YY}
\hline\thickhline
\rowcolor{gray!20}
\textbf{Objective} & \textbf{LeWM} & \textbf{\ours{}} \\
\hline\hline
$\text{One-step}$ & 82.7 & 84.4 \\
$\text{Rollout}$ & 77.8 & \textbf{92.7} \\
\hline
$\Delta$ & \textbf{-4.9}  & \textbf{+8.3} \\
\hline\thickhline
\end{tabularx}
\end{table}

\subsection{Planning Reliability: Failure Anatomy on OGBench-Cube}
\label{sec:failure-anatomy}

The largest planning gap occurs on OGBench-Cube. We examine a separate paired
evaluation of $150$ episodes per model, with identical initial states and goals,
as described in \Cref{app:anatomy}.

\paragraph{Mismatch Failures and Unflagged Failures.}
For each failed episode, we compare two model-predicted terminal costs:
$\hat c$ is the mean elite cost at the final CEM iteration before execution,
and $c$ is the mean elite cost at the first CEM iteration when replanning
after execution and re-observation. Both are computed by the same model, at different planning calls and optimization stages, as detailed in \Cref{app:anatomy}.
We classify failures using $r=c/\hat c$: those with $r>5$ are \textit{mismatch failures}, and those with $r\leq5$ are \textit{unflagged failures}. The threshold is selected within a gap in LeWM's seed-$1$ ratio distribution and held fixed for both models and the remaining seeds. Of the $52$ LeWM failures, $35$ are mismatch failures.

\paragraph{Mismatch Failures Are Sharply Reduced.}
Across all $150$ paired episodes per model, the mismatch failure rate decreases from $23.3\%$ for LeWM to $2.0\%$ for \ours{}, while the unflagged failure rate decreases from $11.3\%$ to $4.0\%$, as shown in \Cref{fig:failure}. Most of the reduction in failures therefore occurs in the mismatch category. Across the 150 paired episodes, \ours{} succeeds on $43$ episodes where LeWM fails, and LeWM succeeds on no episode where \ours{} fails. These results connect the planning gain to fewer failures in which costs rise sharply after the model observes the execution outcome.

\begin{figure*}[t]
  \centering
  \includegraphics[width=\linewidth]{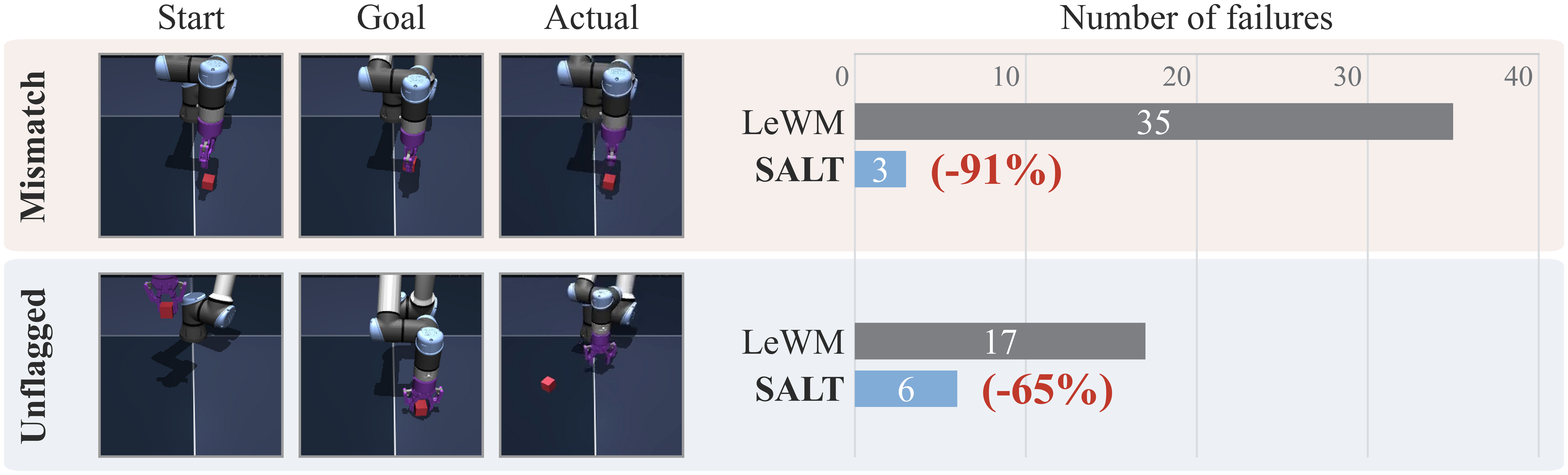}
  \caption{
  \textbf{Failure anatomy on OGBench-Cube.} Among failed episodes, \textit{mismatch failures} satisfy $r=c/\hat c>5$, while \textit{unflagged failures} satisfy $r\leq5$. Here, $\hat c$ and $c$ are model-predicted terminal costs before execution and during subsequent replanning, respectively. The examples are reproduced LeWM failures, and all images are real observations. Bars report counts over $150$ paired episodes per model.
  }
  \label{fig:failure}
  \vspace{-6pt}
\end{figure*}

\subsection{Long-Horizon Stress Test}
\label{sec:long-horizon}

We test whether the planning advantage persists at horizons $H\in\{10,15,20\}$, executing the full planned sequence before each new observation. This stress test uses goals sampled $100$ environment steps after the initial state and an interaction budget of $300$ environment steps. The full protocol is given in \Cref{app:long-horizon}. As shown in \Cref{fig:horizon}, \ours{} achieves higher success in all twelve combinations of environment and horizon. At $H=20$, its advantage is $11$, $8$, $4$, and $17$ percentage points on Two-Room, Reacher, PushT, and OGBench-Cube, respectively. \ours{} therefore retains its planning advantage when the transition is applied recursively over longer intervals without new observations.

\begin{figure*}[t]
  \centering
  \includegraphics[width=\linewidth]{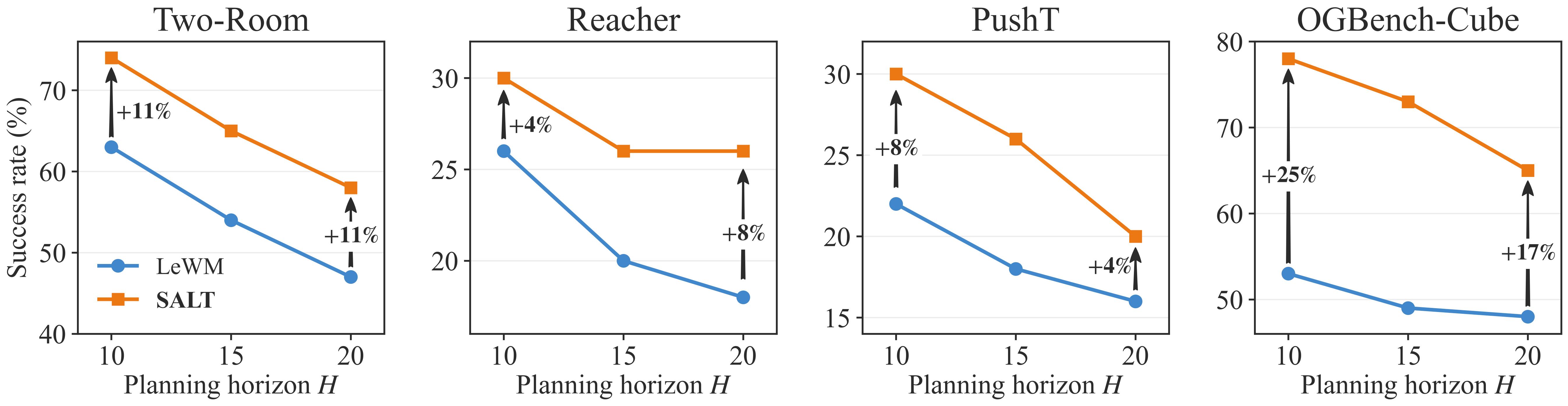}
  \caption{
  \textbf{Long-horizon planning performance.}
  Success rate on the four environments at $H\in\{10,15,20\}$.
  Arrows mark the performance difference between \ours{} and LeWM at $H=10$
  and $H=20$.}
  \label{fig:horizon}
  \vspace{-6pt}
\end{figure*}

\subsection{Planning with Multiple Objects}
\label{sec:cube-double}
We next test whether the planning advantage extends to goals involving multiple objects. Cube-Double \citep{ogbench_park2025} uses the same robot arm and action space as OGBench-Cube but requires two cubes to reach their target poses simultaneously. Both models are trained using the same protocol as for OGBench-Cube, with three independent 
runs and $50$ evaluation episodes per run.

\begin{table}[!htbp]
\centering
\begin{minipage}{0.46\linewidth}
\centering
\small
\setlength\tabcolsep{4pt}
\renewcommand\arraystretch{1.3}
\caption{\textbf{Success rate (\%) on Cube-Double.}}
\label{tab:double}
\begin{tabularx}{\linewidth}{r||Y}
\hline\thickhline
\rowcolor{gray!20}
Method & Success rate \\
\hline\hline
Random & 32.7{\scriptsize$\pm$4.2} \\
LeWM & 36.0{\scriptsize$\pm$5.3} \\
\hline
\rowcolor[HTML]{D7F6FF}
\textbf{\ours{}} & \textbf{54.0}{\scriptsize$\pm$3.5} \\
\hline\thickhline
\end{tabularx}
\end{minipage}
\end{table}

\ours{} achieves a success rate of $54.0\%$, compared with $36.0\%$ for LeWM and $32.7\%$ for the random policy. The $18.0$ percentage point gain over LeWM shows that the planning advantage extends to a task requiring both object goals to be satisfied within the same episode.
\section{Conclusion}
\label{sec:conclusion}


We studied why one-step prediction accuracy can fail to reflect closed-loop
planning quality. Our analysis separates errors introduced at individual
transitions from their subsequent propagation and characterizes state-affine
dynamics by state-independent Jacobians. \ours{} combines this structure with
action conditioning and recursive rollout training. Across four visual planning
environments, it achieves higher mean closed-loop success than reproduced LeWM
despite larger one-step prediction error. The measured propagation operators
exhibit milder amplification over longer rollouts. On OGBench-Cube, \ours{}
reduces failures characterized by sharply higher model-predicted costs after
execution and re-observation.

\FloatBarrier
\bibliography{reference}
\clearpage
\appendix
\appendix

\section{Theoretical Details}
\label{app:theory}

\subsection{Proof of the Multi-Step Error Decomposition}
\label{app:decomposition}

We use the notation of \Cref{sec:error-propagation} and the assumptions of \Cref{prop:error_decomposition}.

\begin{proof}
Define the first-order remainder
\begin{equation}
  r_t := F(z_{t-1}+\delta_{t-1},c_{t-1})
  -F(z_{t-1},c_{t-1})-J_t\delta_{t-1}.
  \label{eq:app-remainder}
\end{equation}
Using $\hat z_{t-1}=z_{t-1}+\delta_{t-1}$ and adding and subtracting $F(z_{t-1},c_{t-1})$, we obtain
\begin{align}
  \delta_t
  &=F(z_{t-1}+\delta_{t-1},c_{t-1})-F(z_{t-1},c_{t-1})
    +\varepsilon_t \notag\\
  &=J_t\delta_{t-1}+r_t+\varepsilon_t.
  \label{eq:app-recursion}
\end{align}
Since $\delta_0=0$, this recursion yields
\begin{equation}
  \delta_H=\sum_{j=1}^{H}\Phi_{j\to H}(\varepsilon_j+r_j),
  \qquad \Phi_{j\to H}=J_HJ_{H-1}\cdots J_{j+1},
  \qquad \Phi_{H\to H}=I.
  \label{eq:app-unroll}
\end{equation}
To verify the expansion, the case $H=1$ follows from $r_1=0$ and $\delta_1=\varepsilon_1$. If \eqref{eq:app-unroll} holds at $H-1$, then
\begin{align}
  \delta_H
  &=J_H\sum_{j=1}^{H-1}\Phi_{j\to H-1}(\varepsilon_j+r_j)
    +(\varepsilon_H+r_H) \notag\\
  &=\sum_{j=1}^{H}\Phi_{j\to H}(\varepsilon_j+r_j),
\end{align}
where $J_H\Phi_{j\to H-1}=\Phi_{j\to H}$ and $\Phi_{H\to H}=I$. This proves \eqref{eq:decomp}.

For the residual bound, assume that $\partial_zF(\cdot,c)$ is $L$-Lipschitz in the state with the same constant $L$ for every action condition $c$. Applying the fundamental theorem of calculus along $z_{t-1}+s\delta_{t-1}$ gives
\begin{equation}
  r_t=\int_0^1
  \big[\partial_zF(z_{t-1}+s\delta_{t-1},c_{t-1})
  -\partial_zF(z_{t-1},c_{t-1})\big]\delta_{t-1}\,ds.
\end{equation}
The Lipschitz assumption bounds the norm of the bracketed term by $Ls\|\delta_{t-1}\|$. Therefore,
\begin{equation}
  \|r_t\|\leq\int_0^1 Ls\|\delta_{t-1}\|^2\,ds
  =\frac{L}{2}\|\delta_{t-1}\|^2.
\end{equation}
\end{proof}

\subsection{Proof of the State-Affine Characterization}
\label{app:affine-characterization}

\begin{proof}[Proof of \Cref{prop:state-affine}]
Fix the action condition $c$ and write $F_c(z)=F(z,c)$.

\emph{(Sufficiency.)} If $F_c(z)=\mathcal A(c)z+\mathcal B(c)$, differentiating with respect to $z$ gives
\[
  \partial_zF_c(z)=\mathcal A(c)
  \qquad\text{for all }z\in\mathbb R^d.
\]
The Jacobian is therefore independent of the state.

\emph{(Necessity.)} Suppose $\partial_zF_c(z)=\mathcal A(c)$ for all $z\in\mathbb R^d$, where $\mathcal A(c)$ does not depend on $z$. Define $G(z)=F_c(z)-\mathcal A(c)z$. Then
\[
  \partial_zG(z)=\partial_zF_c(z)-\mathcal A(c)=0.
\]
For any $z,z'\in\mathbb R^d$, integration along the segment joining them yields
\begin{equation}
  G(z')-G(z)=\int_0^1
  \partial_zG\big(z+s(z'-z)\big)(z'-z)\,ds=0.
\end{equation}
Thus $G$ is constant on $\mathbb R^d$. Denoting this constant by $\mathcal B(c)$ gives $F(z,c)=\mathcal A(c)z+\mathcal B(c)$, completing the proof.
\end{proof}

\section{Implementation Details}
\label{app:detail_setup}

\subsection{Environments and Datasets}
\label{app:environments}

We follow the LeWM benchmark settings \citep{maes_lelidec2026lewm}, including observation and action spaces, success criteria, and training datasets, for Two-Room (2D navigation) \citep{sobal2025stresstesting}, Reacher (two-joint reaching) \citep{tassa2018deepmindcontrolsuite}, PushT (2D pushing) \citep{chi2023diffusionpolicy,zhou2025dinowm}, and OGBench-Cube (3D robot-arm manipulation) \citep{ogbench_park2025}. We split sampled windows into $90\%$ for training and $10\%$ for validation using split seed $3072$. Evaluation initial states and goals are sampled from the offline dataset. Cube-Double \citep{ogbench_park2025}, described in \Cref{sec:cube-double}, follows the OGBench-Cube training and evaluation settings unless stated otherwise.

\subsection{Training Configuration}
\label{app:training}

Unless otherwise stated, \ours{} is trained for $10$ epochs with AdamW, a learning rate of $5\times10^{-5}$, weight decay $10^{-3}$, and an epoch-based linear-warmup cosine-annealing schedule. We use batch size $128$, \texttt{bf16} precision, and gradient clipping at $1.0$. The latent dimension is $d=192$. SIGReg uses weight $\lambda=0.09$, $17$ knots, and $1024$ projections.

The visual encoder is a ViT-Tiny trained from scratch, with $12$ layers, width $192$, $3$ attention heads, feedforward dimension $768$, and patch size $14$. Its CLS-token output is passed through a $192\to2048\to192$ projection head with batch normalization and GELU to produce the regularized latent state $z$. The action encoder maps each action block to dimension $192$; its input size depends on the frameskip and action dimension. Components and settings shared with reproduced LeWM are listed in \Cref{app:baselines}.

\subsection{Transition Model}
\label{app:implementation}

The transition in \eqref{eq:transition} uses $R=16$ dense modulation matrices with no low-rank factorization:
\begin{equation}
  F(z_t,c_t)=
  \left[\mathcal A_0+\sum_{r=1}^{R}(W_gc_t)_rN_r\right]z_t+Bc_t+b,
  \qquad c_t\in\mathbb R^{192}.
  \label{eq:app-transition}
\end{equation}
It acts directly in the regularized latent space, without the predictor projection head used by LeWM. We parameterize the shared base matrix as
\begin{equation}
  \mathcal A_0=U\operatorname{diag}(\tanh s)V^\top,
  \qquad U=\exp(W_u-W_u^\top),
  \qquad V=\exp(W_v-W_v^\top).
\end{equation}
The skew-symmetric generators make $U$ and $V$ orthogonal, so the singular values of $\mathcal A_0$ are $|\tanh s_i|<1$ for finite parameters. This constrains the base matrix, not the full action-conditioned matrix $\mathcal A(c)$.

\Cref{tab:params} lists all stored parameters and their initialization. These values give $U=V=I$, $\mathcal A_0\approx0.9951I$, and an initial prediction $F(z_t,c_t)\approx0.9951z_t$. The gate $W_g$ has no bias.

\begin{table}[H]
\centering
\caption{\textbf{Parameters of the \ours{} transition model} ($d=192$, $R=16$). Counts are stored parameters. Only the skew-symmetric parts of $W_u$ and $W_v$ affect the transition, corresponding to $18{,}336$ independent entries for each matrix.}
\label{tab:params}
\begin{tabularx}{\textwidth}{l l c c >{\centering\arraybackslash}X}
\toprule
\textbf{Component} & \textbf{Symbol} & \textbf{Shape} & \textbf{Parameters} & \textbf{Initialization} \\
\midrule
Generator of $U$ & $W_u$ & $192\times192$ & $36{,}864$ & $0$ \\
Generator of $V$ & $W_v$ & $192\times192$ & $36{,}864$ & $0$ \\
Singular-value parameters & $s$ & $192$ & $192$ & $3.0$ \\
Gate & $W_g$ & $16\times192$ & $3{,}072$ & $\mathcal{U}(-1/\sqrt{192},\,1/\sqrt{192})$ \\
Modulation modes & $N_1,\dots,N_{16}$ & $16\times192\times192$ & $589{,}824$ & $0$ \\
Action injection & $B$ & $192\times192$ & $36{,}864$ & $0$ \\
Bias & $b$ & $192$ & $192$ & $0$ \\
\midrule
\textbf{Total} & & & $\mathbf{703{,}872}$ & \\
\bottomrule
\end{tabularx}
\end{table}

\subsection{Training Procedure}
\label{app:training-procedure}

Algorithm~\ref{alg:salt} summarizes the recursive training procedure. A window contains six observations sampled every five environment steps and therefore spans $25$ environment steps. Each intervening action block concatenates five raw actions. The first latent initializes the rollout, and all $K=5$ predictions receive equal supervision. SIGReg uses all six encoded frames, and target latents are not detached.
\begin{algorithm}[!htbp]
\caption{Mini-batch training objective for \ours{}.}
\label{alg:salt}
\begin{lstlisting}[style=salt,basicstyle=\ttfamily\footnotesize]
def affine_step(z, c):
    # z: (B, D), c: (B, D)
    U = expm(W_u - W_u.T)
    V = expm(W_v - W_v.T)

    # Shared base transition: sigma_max(A0) < 1
    A0 = U @ diag(tanh(s)) @ V.T

    # Action-conditioned transition
    g   = c @ W_g.T
    A_c = A0 + einsum('br,rij->bij', g, N)

    return einsum('bij,bj->bi', A_c, z) \
           + c @ B.T + b


def loss(obs, act, lam=0.09):
    # obs: (B, K+1, C, H, W)
    # act: (B, K, ...), action blocks

    # Encode all frames for initialization, targets, and SIGReg
    z = encoder(obs)              # (B, K+1, D)
    c = action_encoder(act)       # (B, K, D)

    # Recursive rollout initialized from the first true latent
    z_cur = z[:, 0]
    preds = []

    for k in range(K):
        z_cur = affine_step(z_cur, c[:, k])
        preds.append(z_cur)

    z_hat = stack(preds, dim=1)    # (B, K, D)

    # Targets are not detached
    loss_roll = mse(z_hat, z[:, 1:])
    loss_reg  = sigreg(z)

    return loss_roll + lam * loss_reg
\end{lstlisting}
\end{algorithm}

\subsection{Planner and Evaluation Protocol}
\label{app:planner}

We use the planning objective and default evaluation protocol in \Cref{sec:problem-setup,sec:setup}. CEM uses $300$ candidate action sequences per iteration, $30$ optimization iterations, $30$ elites, and initial sampling variance $1.0$. Each model action spans five environment steps, so the default horizon and replanning interval $H=5$ correspond to $25$ environment steps.

The evaluation interface retains LeWM's three-frame history window; \ours{} uses only the most recent latent state. Each evaluation seed specifies $50$ episodes. Matched seeds give paired models identical initial states and goals. The random policy is averaged over three runs.

\subsection{Efficiency Measurement}
\label{app:efficiency}

All timings in \Cref{tab:efficiency} are measured consecutively on the same NVIDIA H100 $80$\,GB GPU. Single-step timing uses a batch of $300$ latent states and covers the predictor and its projection head where present. Planning time measures one steady-state batched CEM solve; the reported summary is the unweighted mean over the four environments and $H\in\{5,10,15,20\}$. Training throughput is measured separately for the one-step and rollout configurations.

\subsection{Baseline Provenance}
\label{app:baselines}

Random, reproduced LeWM, and \ours{} are evaluated in our pipeline. LeWM and \ours{} are trained in our implementation with matched visual encoders, encoder projection heads, action encoders, latent dimensions, optimizer settings, batch sizes, and CEM planners. LeWM retains its Transformer predictor, predictor projection head, one-step prediction objective, and SIGReg. All paired comparisons use this reproduction.

The PLDM \citep{sobal2026learning}, DINO-WM \citep{zhou2025dinowm}, Sub-JEPA \citep{zhao2026subjepa}, SD-JEPA \citep{thil2026sdjepa}, and SMWM \citep{ivashkov2026sensorimotor} rows in \Cref{tab:main} are taken from their respective papers. Their training, hyperparameter selection, and reporting protocols differ, so these rows provide context rather than controlled comparisons.

\FloatBarrier
\section{Empirical Validation of Error Propagation}
\label{app:error-validation}

\subsection{Measurement Protocol}
\label{app:decomp-protocol}

The PushT diagnostics in \Cref{app:decomp-recon,app:operator-growth,app:spectral-control} use samples drawn from
the training dataset with a fixed generator seed. Rollouts start from an
encoded latent state, so $\delta_0=0$. Jacobians and matrix products are
computed in float64.

We follow the notation of Section~\ref{sec:error-propagation}. For \ours{}, the
propagation operator is constructed directly from the learned transition
matrices $\mathcal{A}(c)$. For reproduced LeWM, the predictor is conditioned on a
three-frame history window. We differentiate only with respect to the most
recent latent-state channel while holding the earlier history frames fixed.
The resulting $d\times d$ Jacobian therefore does not represent the complete
augmented-state Jacobian of the history-based predictor.

\subsection{Rollout Reconstruction}
\label{app:decomp-recon}

We reconstruct the terminal rollout error using the first-order propagation terms
\begin{equation}
  \tilde\delta_H=\sum_{j=1}^{H}\Phi_{j\to H}\,\varepsilon_j,
  \qquad
  \rho_H=\frac{\|\delta_H-\tilde\delta_H\|}{\|\delta_H\|}.
\end{equation}

For \ours{}, the state-affine transition satisfies the single-state decomposition
derived in Section~\ref{sec:state-affine}, so $\rho_H=0$ in exact arithmetic;
nonzero values therefore reflect finite numerical precision.

For reproduced LeWM, only the current-state channel is linearized while the
earlier history channels are held fixed. Consequently,
$\rho_H$ may contain both nonlinear current-state effects and contributions
propagated through the history channels. It should therefore not be
interpreted as the nonlinear Taylor remainder alone.

Across $H\in\{1,\ldots,20\}$, the \ours{} discrepancy remains at the
$10^{-6}$ level. The corresponding discrepancy for reproduced LeWM increases
with rollout horizon and reaches 49.4\% at $H=20$.

\subsection{Propagation-Operator Growth}
\label{app:operator-growth}

We measure the spectral norm of the propagation operator from the initial state to rollout step $k$ on PushT:
\begin{equation}
  \Phi_{0\rightarrow k}=J_kJ_{k-1}\cdots J_1.
\end{equation}
Measurements are aggregated over nine window--batch combinations obtained from three trajectory-window lengths $\{20,25,30\}$ and three batches. Each \ours{} batch contains 40 trajectories. Each LeWM batch contains 16 trajectories because Jacobian computation is substantially more expensive. The reported center is the geometric mean, and the uncertainty band spans the minimum and maximum across the nine measurements.

Between rollout steps $1$ and $20$, the geometric-mean norm changes from
$2.63$ to $11.26$ for \ours{} and from $2.46$ to $21.51$ for reproduced LeWM. These endpoint values underlie the growth factors reported in
Section~\ref{sec:better-planner}.

The spectral norm is a worst-case amplification measure. Step-wise operator
norms should therefore not be multiplied to estimate the realized
amplification of an individual prediction error.

\subsection{Control for the Spectral Constraint}
\label{app:spectral-control}

To test whether the constraint $\sigma_{\max}(\mathcal{A}_0)<1$ is responsible for the observed propagation behavior, we train an otherwise identical state-affine model in which $\mathcal{A}_0$ is represented directly as an unconstrained dense matrix initialized to $0.9951I$. All remaining parameters and training settings are unchanged.
We then repeat the diagnostic in \Cref{app:operator-growth}.

Removing the constraint changes the measured twenty-step propagation growth
from $4.28\times$ to $4.42\times$, compared with $8.74\times$ for reproduced
LeWM. 

These measurements do not imply that \ours{} rollouts are globally contractive.

\subsection{Decoded Open-Loop Rollouts}
\label{app:decoded-rollouts}

For each environment--checkpoint pair, we freeze the world model and train a post-hoc decoder $D:\mathbb R^{192}\to\mathbb R^{3\times224\times224}$ on encoded ground-truth observations and their corresponding images. It uses $196$ learnable query tokens on a $14\times14$ grid, four cross-attention blocks with width $512$ and eight heads, and a linear output projection to $16\times16$ RGB patches. Training uses pixel MSE for $30{,}000$ steps with the world model's image preprocessing. The decoder is used only for visualization, never for world-model training, planning, or performance measurement.

\Cref{fig:decoded-rollout} uses the \ours{} checkpoint for each environment and the corresponding dataset actions. The displayed predictions retain recognizable agent, arm, or object configurations across the rollout. They illustrate the observation-space content of predicted latents; quantitative conclusions rely on aggregate measurements rather than these examples.

\begin{figure}[!htbp]
  \centering
  \includegraphics[width=\textwidth]{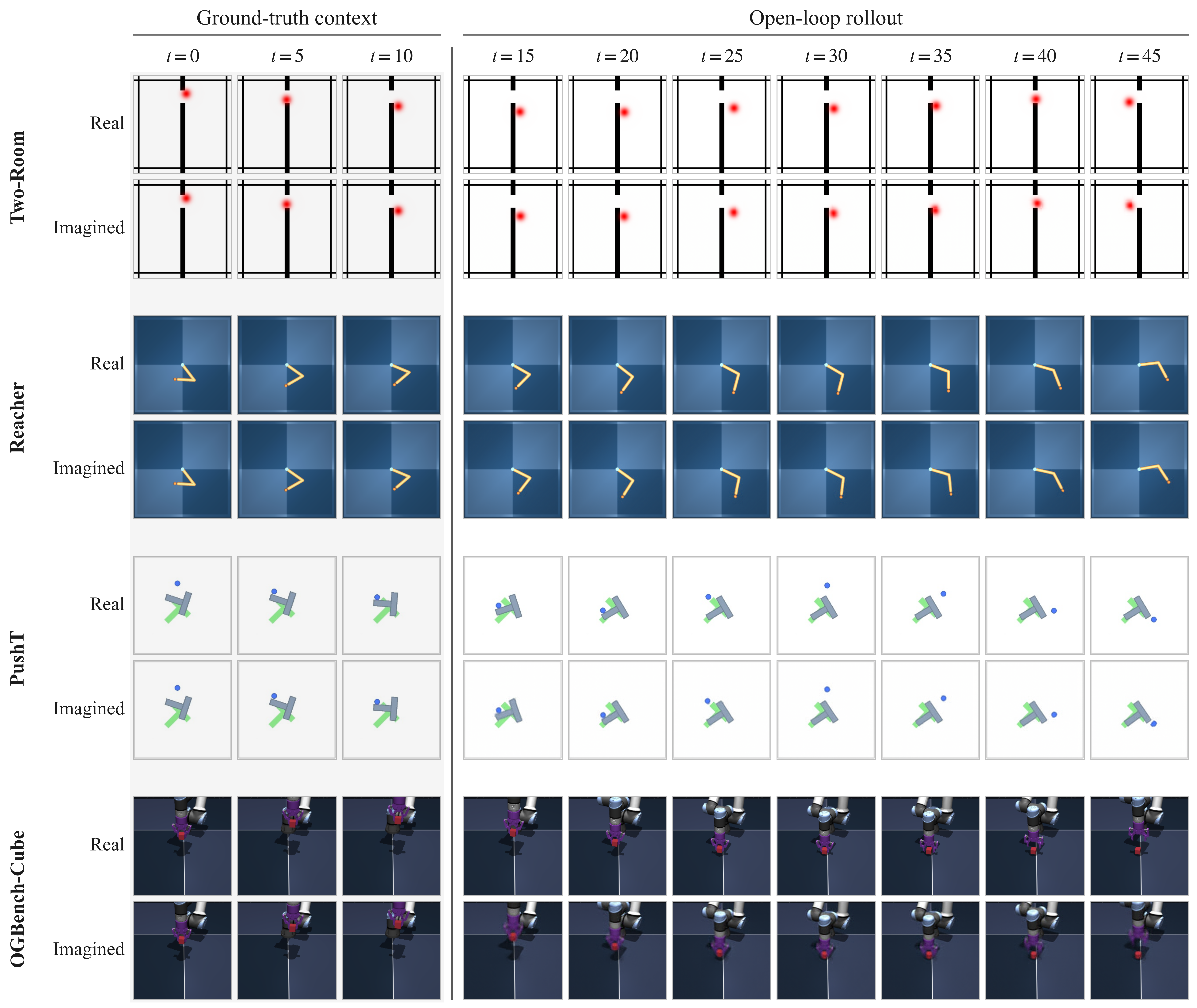}
  \caption{\textbf{Decoded open-loop rollouts across four environments.}
  The first three columns ($t=0,5,10$) are ground-truth context frames from the common evaluation interface. Only the final frame initializes \ours{}, which predicts seven subsequent latent states using dataset actions. At frameskip $5$, these predictions cover $35$ additional environment steps ($t=15,\ldots,45$). Images are produced by a post-hoc decoder trained only on ground-truth latents and excluded from world-model training and planning.}
  \label{fig:decoded-rollout}
\end{figure}
\FloatBarrier

\subsection{Rollout Error and Propagation Across Four Environments}
\label{app:four-env-propagation}

We additionally evaluate rollout error and propagation-operator norms on
Two-Room, Reacher, PushT, and OGBench-Cube using test seeds $1$, $2$, and $3$.
These measurements complement the nine-measurement PushT summary in
\Cref{app:operator-growth}.
As shown in \Cref{fig:core-4env}, \ours{} has larger initial prediction error
in all four environments, while its measured propagation norms are generally
lower at intermediate and later rollout steps.
On Reacher, the \ours{} norm decreases after an initial rise. On PushT, it
continues to grow but more slowly than that of LeWM.

Smaller propagation norms do not imply smaller rollout errors at every
horizon. On OGBench-Cube, the rollout errors remain close over much of the
horizon, and \ours{} does not have lower error at the final steps shown.
The four-environment results therefore broaden the propagation diagnostic
beyond PushT while preserving the distinction between error amplification
and realized rollout error.

\begin{figure}[!htbp]
  \centering
  \includegraphics[width=\textwidth]{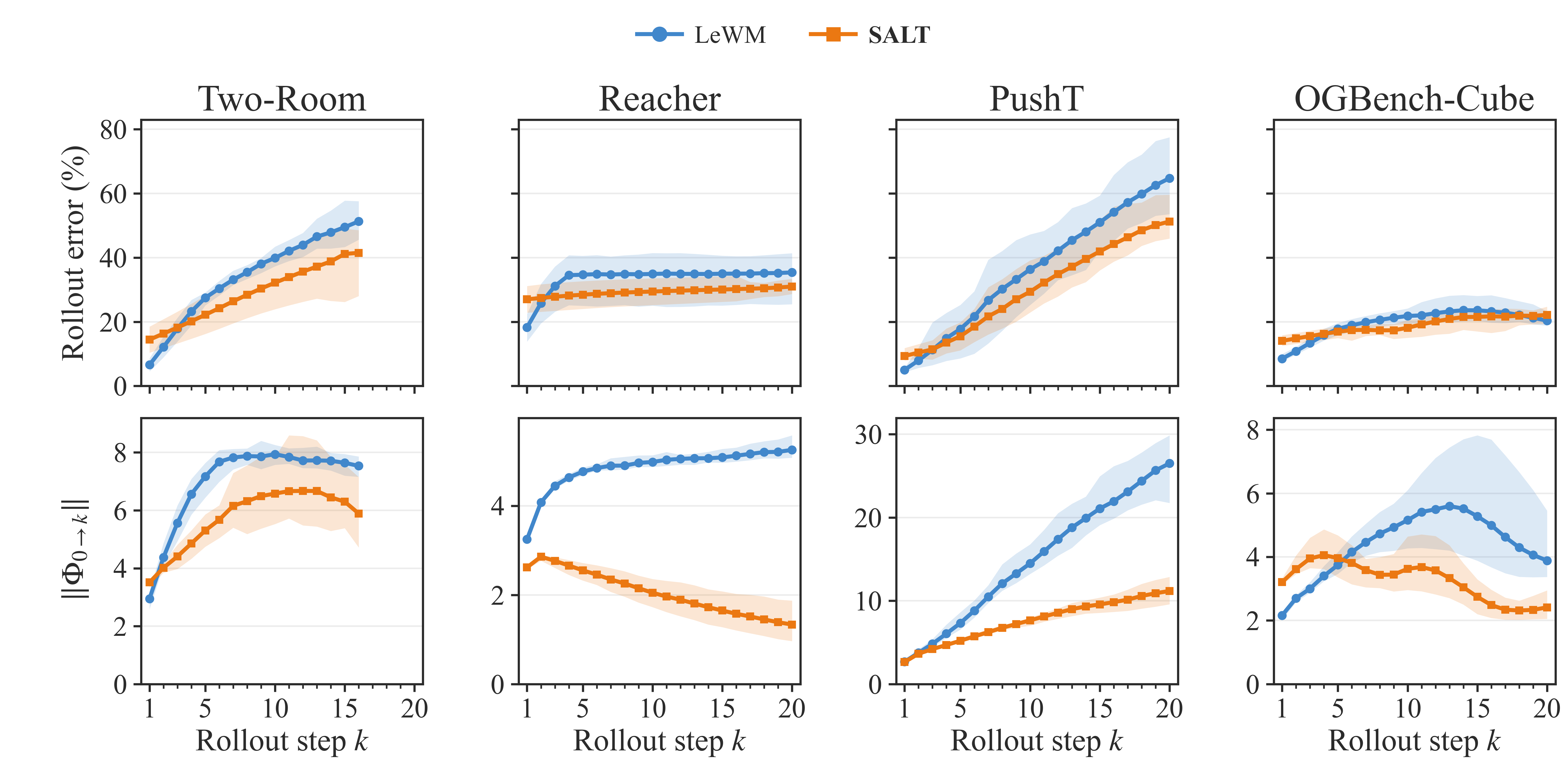}
  \caption{\textbf{Rollout error and propagation across four environments.}
  Top: relative rollout error as a function of rollout step $k$.
  Bottom: the spectral norm of the measured propagation operator
  $\Phi_{0\to k}$.
  Results use test seeds $1$, $2$, and $3$.
  Blue circles denote reproduced LeWM; orange squares denote \ours{}.
  Each model step corresponds to five environment steps.}
  \label{fig:core-4env}
\end{figure}
\FloatBarrier

\section{Failure Anatomy on OGBench-Cube}
\label{app:anatomy}

\subsection{Paired Evaluation and Failure Criterion}
\label{app:anatomy-setup}

This diagnostic is separate from the aggregate evaluation in \Cref{tab:main}. Under the default protocol in \Cref{app:planner}, we evaluate seeds $1$--$3$ with $50$ identical episodes per seed for both models, giving $150$ paired episodes.

For each failed episode, $\hat c$ is the mean elite cost at the final CEM iteration of the first planning call. Following \Cref{sec:failure-anatomy}, we denote by $c$ the mean elite cost at the first CEM iteration of the second call, after executing the first sequence and encoding the resulting observation. Thus $c$ is a post-execution replanning cost; both quantities are computed in the same model's latent space at different planning calls and optimization stages. Their ratio $r=c/\hat c$ is not a direct comparison with a physical ground-truth cost.

We select $r=5$ within the gap in LeWM's seed-$1$ ratio distribution and keep it fixed for both models and the remaining seeds. Failed episodes with $r>5$ are \textit{mismatch} failures; those with $r\leq5$ are \textit{unflagged} failures.

\subsection{Failure Composition}
\label{app:anatomy-results}

\Cref{tab:anatomy-rate} partitions the $52$ LeWM failures and $9$ \ours{} failures. Rates use all $150$ episodes per model, not only failed episodes.

\begin{table}[H]
\centering
\caption{\textbf{Episode-level failure rates by type on OGBench-Cube.}
Rates are computed over $150$ paired episodes per model using the fixed threshold $r=5$. LeWM denotes the reproduced baseline.}
\label{tab:anatomy-rate}
\begin{tabular}{lcc}
\toprule
\textbf{Model} & \textbf{Mismatch} & \textbf{Unflagged} \\
\midrule
LeWM & $35/150$ ($23.3\%$) & $17/150$ ($11.3\%$) \\
\ours{} & $3/150$ ($2.0\%$)  & $6/150$ ($4.0\%$) \\
\bottomrule
\end{tabular}
\end{table}

Across the paired episodes, both models succeed on $98$, only \ours{} succeeds on $43$, only LeWM succeeds on none, and neither succeeds on $9$. These counts give success rates of $94.0\%$ for \ours{} and $65.3\%$ for LeWM in this evaluation, as reported in \Cref{sec:failure-anatomy}.

\subsection{Figure Details}
\label{app:anatomy-figure}

In \Cref{fig:failure}, the start and goal images are dataset observations; the post-execution image is recorded after the first planned sequence. All are real observations. Both representative examples are failures of reproduced LeWM from evaluation seed $1$, with ratios near the seed-$1$ median of their respective failure types.

\section{Long-Horizon Planning}
\label{app:long-horizon}

\subsection{Evaluation Protocol}
\label{app:long-horizon-protocol}

We use the stress-test protocol in \Cref{sec:long-horizon}: a target offset of $100$ environment steps, a $300$-step interaction budget, and planning horizons and replanning intervals $H\in\{10,15,20\}$. With frameskip $5$, the full-sequence execution intervals are $50$, $75$, and $100$ environment steps. Other settings follow \Cref{app:planner}; LeWM denotes the reproduced baseline.

\subsection{Results Across Four Environments}
\label{app:long-horizon-results}
\Cref{tab:horizon-main} lists the success rates underlying \Cref{fig:horizon}. \ours{} exceeds LeWM in all twelve environment–horizon combinations, with gaps between $4$ and $25$ percentage points.

\begin{table}[H]
\centering
\caption{\textbf{Success rate (\%) under the long-horizon protocol.} Results correspond to the long-horizon experiment reported in the main text.}
\label{tab:horizon-main}
\begin{tabular}{lcccc}
\toprule
\textbf{Environment} & \textbf{Horizon} & \textbf{LeWM} & \textbf{\ours{}} & \textbf{Gap (pp)} \\
\midrule
Two-Room
& $10$ & $63$ & $74$ & $\mathbf{+11}$ \\
& $15$ & $54$ & $65$ & $\mathbf{+11}$ \\
& $20$ & $47$ & $58$ & $\mathbf{+11}$ \\
\midrule
Reacher
& $10$ & $26$ & $30$ & $\mathbf{+4}$ \\
& $15$ & $20$ & $26$ & $\mathbf{+6}$ \\
& $20$ & $18$ & $26$ & $\mathbf{+8}$ \\
\midrule
PushT
& $10$ & $22$ & $30$ & $\mathbf{+8}$ \\
& $15$ & $18$ & $26$ & $\mathbf{+8}$ \\
& $20$ & $16$ & $20$ & $\mathbf{+4}$ \\
\midrule
OGBench-Cube
& $10$ & $53$ & $78$ & $\mathbf{+25}$ \\
& $15$ & $49$ & $73$ & $\mathbf{+24}$ \\
& $20$ & $48$ & $65$ & $\mathbf{+17}$ \\
\bottomrule
\end{tabular}
\end{table}

\section{Ablations and Robustness}
\label{app:ablations}

Each ablation compares variants within the same evaluation round. Absolute values across different ablation groups should therefore not be interpreted as directly comparable.

\subsection{Action-Conditioned Transition}
\label{app:action-ablation}

We replace $\mathcal A(c)$ with the shared base matrix $\mathcal A_0$ while retaining $Bc+b$. Actions therefore remain inputs but no longer modulate the state-transition matrix.

\begin{table}[H]
\centering
\caption{\textbf{Effect of removing the action-dependent modulation of the transition matrix.} Success rate (\%) on Reacher over three evaluation seeds.}
\label{tab:action-modulation}
\begin{tabular}{lc}
\toprule
\textbf{Configuration} & \textbf{Planning success} \\
\midrule
Action-independent transition matrix & $11.3 \pm 3.1$ \\
Full structure                       & $90.7 \pm 3.1$ \\
\bottomrule
\end{tabular}
\end{table}

Removing action-dependent modulation substantially reduces success on Reacher; the additive action term alone does not match the full model in this setting.

\begingroup
\setlength{\parskip}{2pt plus 1pt minus 1pt}
\captionsetup[figure]{font=small}
\Needspace{0.5\textheight}
\subsection{Robustness}
\label{app:robustness}

\subsubsection{Regularizer Weight}
\label{app:lambda}

We evaluate \ours{} across SIGReg weights on PushT. As shown in \Cref{fig:lambda}, performance remains stable near the default $\lambda=0.09$ and decreases toward the tested extremes.

\begin{figure}[H]
\centering
\includegraphics[width=0.36\linewidth]{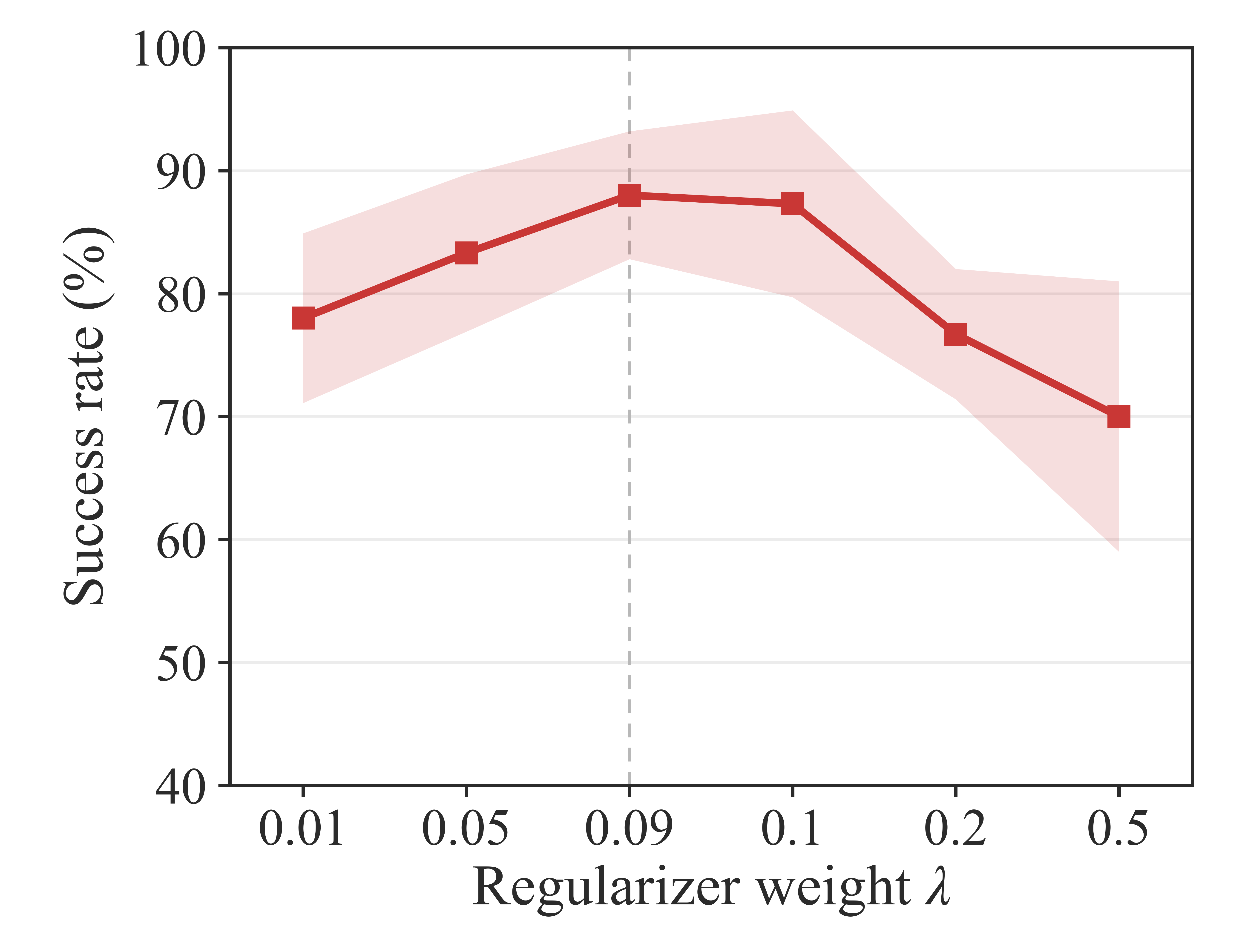}
\caption{\textbf{Success rate against the SIGReg weight on PushT.} Shading shows standard deviation across three evaluation seeds; the dashed line marks the default $\lambda=0.09$.}
\label{fig:lambda}
\end{figure}
\FloatBarrier

\subsubsection{Latent Dimensionality}
\label{app:latent-dim}

We keep the visual backbone fixed and adjust the projection-head output, action-condition dimension, and transition parameters to match $d$. As shown in \Cref{fig:latent-dim}, success varies by at most approximately three percentage points in each environment over the tested fourfold range.

\begin{figure}[H]
\centering
\includegraphics[width=0.36\linewidth]{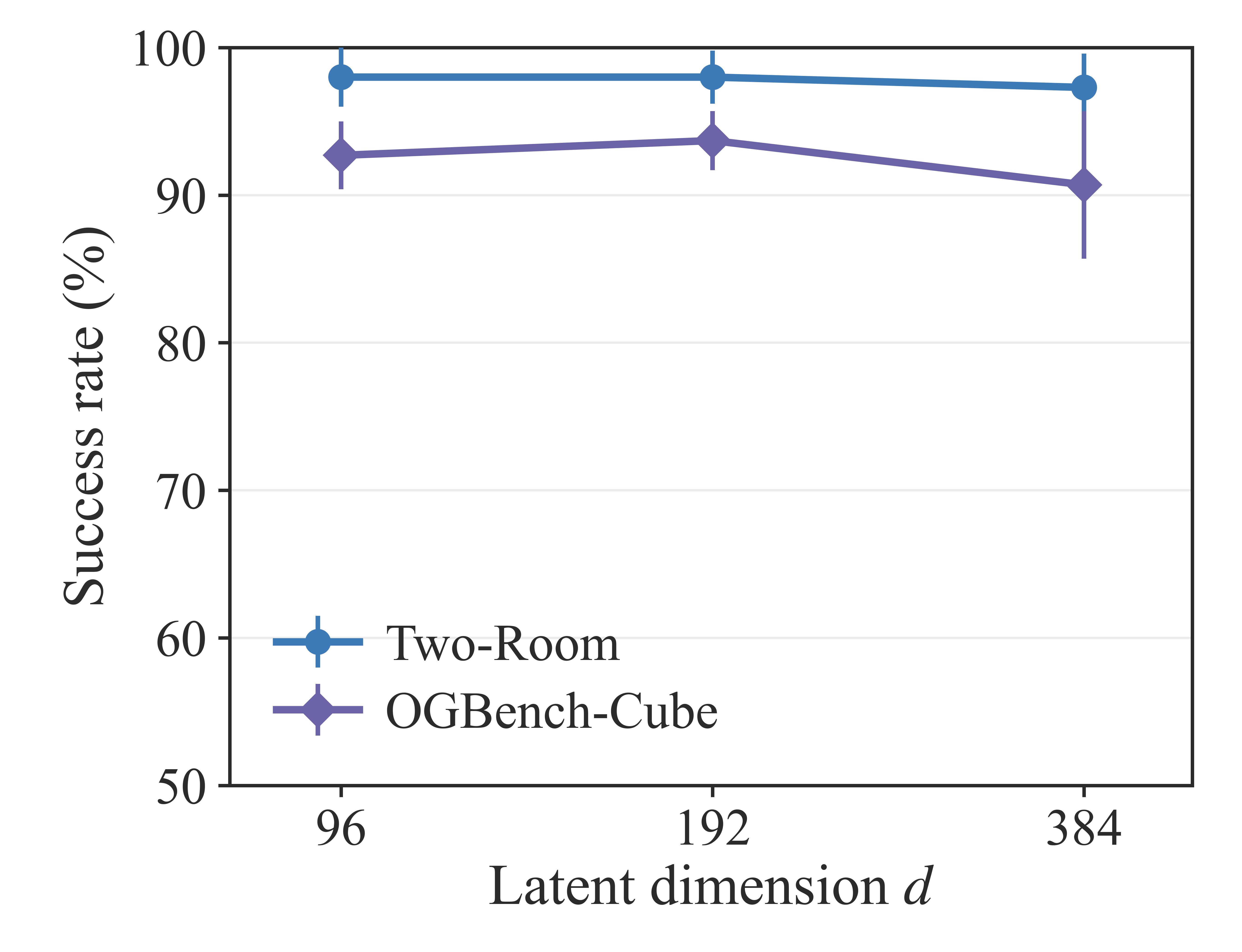}
\caption{\textbf{Success rate of \ours{} against latent dimension on Two-Room and OGBench-Cube.} The horizontal axis uses a base-two log scale; error bars show standard deviation across three evaluation seeds.}
\label{fig:latent-dim}
\end{figure}
\FloatBarrier

\section{Limitations and Scope}
\label{app:limitations}

\paragraph{Scope of the model family.}
We study JEPA-style world models that plan through recursive latent prediction. Architectures with different prediction and control interfaces, including vision-language-action models, are outside this scope.

\paragraph{Scope of the analysis.}
Our theoretical and empirical analyses address recursive error propagation, one aspect of latent planning; they do not fully explain all factors determining planning performance.

\endgroup

\end{document}